\documentclass[conference,letterpaper]{IEEEtran}

\usepackage[utf8]{inputenc} 
\usepackage[T1]{fontenc}
\usepackage{url}
\usepackage{ifthen}
\usepackage{cite}
\usepackage[cmex10]{amsmath} 

\newtheorem{assumption}{Assumption}
\newtheorem{theorem}{Theorem}
\newtheorem{lemma}{Lemma}
\newtheorem{definition}{Definition}
\newtheorem{corollary}{Corollary}
\newtheorem{remark}{Remark}

\newtheorem{proof}{Proof}

\usepackage{amssymb}

\usepackage{algpseudocode}
\usepackage{algorithm}

\usepackage[sort&compress,numbers]{natbib}

\usepackage{graphicx}
\usepackage{float}
\usepackage{subfigure}

\begin{document}
\title{A Unified Regularization Approach to High-Dimensional Generalized Tensor Bandits} 


\author{%
  \IEEEauthorblockN{Jiannan Li, Yiyang Yang and Yao Wang}
  \IEEEauthorblockA{School of Mangement \\
                    Xi’an Jiaotong University\\
                    Xi’an, China\\
                    Email: jiannanli@stu.xjtu.edu.cn, \{yyyang817,yao.s.wang\}@gmail.com}
  \and
  \IEEEauthorblockN{Shaojie Tang}
  \IEEEauthorblockA{Department of Management Science and Systems\\ 
                    University at Buffalo\\
                    Buffalo, New York, USA\\
                    Email: shaojiet@buffalo.edu}
}

\maketitle


\begin{abstract}
   Modern decision-making scenarios often involve data that is both high-dimensional and rich in higher-order contextual information, where existing bandits algorithms fail to generate effective policies. In response, we propose in this paper a generalized linear tensor bandits algorithm designed to tackle these challenges by incorporating low-dimensional tensor structures, and further derive a unified analytical framework of the proposed algorithm. Specifically, our framework introduces a convex optimization approach with the weakly decomposable regularizers, enabling it to not only achieve better results based on the tensor low-rankness structure assumption but also extend to cases involving other low-dimensional structures such as slice sparsity and low-rankness. The theoretical analysis shows that, compared to existing low-rankness tensor result, our framework not only provides better bounds but also has a broader applicability. Notably, in the special case of degenerating to low-rank matrices, our bounds still offer advantages in certain scenarios.

\end{abstract}

\section{Introduction}

Contextual bandits (CB) is a sequential decision-making problem where, in each round, the agent leverages available contextual information to guide its choices—deciding whether to explore new, unexplored arms or exploit previously selected arms in order to maximize cumulative reward. Due to its ability to make better decisions in complex and dynamic environments, CB have gained increasing attention. They have been widely applied in various domains, including recommendation systems \citep{bastani2022learning,aramayo2023multiarmed}, dynamic pricing \citep{wang2021multimodal,luo2024distribution}, personalized healthcare \citep{bastani2020online,zhou2023spoiled}, and other areas \citep{grant2021filtered,agrawal2023tractable}. However, it is worth noting that the rewards are not necessarily continuous variables in the aforementioned sequential decision-making problems. For example, binary rewards in recommendation systems and count-based rewards in ad searches. This highlights the need for more flexible generalized linear models.

Meanwhile, with the rapid advancement of modern information technology, the quantity and diversity of available feature information have increased significantly. This shift implies that, to improve the accuracy of decision making, more complex interaction factors must be considered. At the same time, the dimensions of these factors have expanded rapidly, leading to the emergence of high-order, high-dimensional problems. We illustrate the existence of this phenomenon through specific application cases. For example, traditional personalized healthcare schemes focus primarily on simple interactions between diseases and drugs. However, with the adoption of targeted therapies, complex associations between drugs, targets, and diseases have become central to research, involving thousands or even tens of thousands of features related to diseases and drugs \citep{bayati2014data,razavian2015population}. Similarly, to improve click-through rates, online recommendation platforms need to process high-dimensional data, including user and product features \citep{naik2008challenges}, while also addressing complex interactions between users, products and time. In these scenarios, assumptions based on low-order and moderately sized features \citep{li2010contextual,abbasi2011improved} are no longer applicable, necessitating a shift in research focus towards high-order, high-dimensional problems.

In this paper, we propose a unified regularization approach to handle high-dimensional tensor bandits problems under generalized linear reward relationships, incorporating various low-dimensional structures. Specifically, considering that tensors have complex algebraic structures, certain low-dimensional structures (i.e., sparsity \cite{kim2019doubly,ren2024dynamic} and low-rankness \cite{li2022simple,cai2023doubly}) may appear at the entry level, fiber level, or slice level, rather than being confined to the whole tensor. As such, the goal of this work is to leverage low-dimensional knowledge to facilitate bandit learning and computation. Our key contributions are summarized as follows:


\begin{itemize}
    \item We propose a unified algorithmic framework, that is, G-ELTC, for high-dimensional generalized tensor bandits problems. This framework accommodates various low-dimensional structural constraints, such as slice sparsity and low-rankness, while providing a structure-independent regret analysis for the algorithm.
    \item Considering the flexibility and expressiveness of Tucker decomposition in handling high-order tensors \cite{kolda2009tensor}, we derive improved bounds on dimensionality and rank under low multi-linear rank tensors and demonstrate the superiority and effectiveness of our framework in high-dimensional and high-order scenarios.    
    \item Our theoretical results not only encompass the relevant findings of low-rank matrix bandits under linear models but also extend them to generalized linear relationships, thereby broadening their applicability to a wider range of decision-making scenarios.
    \item  To control the parameter estimation error and further derive the regret analysis for the generalized linear tensor settings, we innovatively introduce the generic chaining technique from stochastic processes, which fundamentally differs from the concentration inequality-based methods in the literature \cite{raskutti2019convex}, thus providing a new theoretical perspective and solution approach to this problem.
    
    
      
\end{itemize}

\section{Preliminaries and Notations}
To enhance readability, this section provides explanations of the symbols and related definitions involved.

We summarize the notations used in this paper as follows. Scalars, vectors, matrices, tensors, and sets are represented by the symbols \(a\), \(\boldsymbol{a}\), \(A\), \(\mathcal{A}\), and \(\mathbb{A}\), respectively. In particular, for the $N$-order tensor $\mathcal{A} \in \mathbb{R}^{d_1 \times d_2 \times \cdots \times d_N}$,
$\mathcal{A}_{i_1, i_2, \cdots, i_N}$ denotes the $\left(i_1, i_2, \cdots, i_N\right)$-th element. $\mathcal{A}_{i_1, i_2, \cdots,i_{n-1}, \cdot,i_{n+1},\cdots,i_N}$ is the mode-$n$ fiber, where $n \in [N]:=\{1,2,\cdots,N\}$. $(n_1,n_2,\cdots,n_n)$-slice is $\mathcal{A}_{i_1, i_2, \cdots,i_{n_1-1}, \cdot,i_{n_1+1},\cdots, i_{n_2-1}, \cdot,i_{n_2+1},\cdots, i_{n_n-1}, \cdot,i_{n_n+1},\cdots,i_N}$.


The representation of the object of this paper is in the form of tensors, and therefore, it will involve some operations related to them \cite{kolda2009tensor}, as detailed below.

\begin{definition}[Tensor inner product]
For the tensor $\mathcal{A},\mathcal{B} \in \mathbb{R}^{d_1 \times d_2 \times \cdots \times d_N}$, their inner product is defined as 
$$
\langle \mathcal{A}, \mathcal{B}\rangle:=\sum_{i_1 \in\left[d_1\right]} \sum_{i_2 \in\left[d_2\right]} \cdots \sum_{i_N \in\left[d_N\right]} \mathcal{A}_{i_1, i_2, \cdots, i_N} \mathcal{B}_{i_1, i_2, \cdots, i_N} .
$$
\label{def1}
\end{definition}

\begin{definition}[Tensor Frobenius norm]
The Frobenius norm of $\mathcal{A}$ is defined as $\|\mathcal{A}\|_F:=\sqrt{\langle \mathcal{A}, \mathcal{A} \rangle}$. 
   \label{def2}
\end{definition}

\begin{definition}[Tensor mode product] For matrix $B \in \mathbb{R}^{d_{n^{\prime}} \times d_n}$ and tensor \(\mathcal{A} \in \mathbb{R}^{d_1 \times d_2 \times \cdots \times d_N}\), the mode-\(n\) (matrix) product \(\mathcal{A} \times_n B\) is defined as an \(N\)-th order tensor with dimensions \(\left(d_1, \cdots, d_{n-1}, d_{n^{\prime}}, d_{n+1}, \cdots, d_N\right)\), and its $(i_1, \cdots, i_{n-1}, i_{n^{\prime}}, i_{n+1}, \cdots, i_N)$-th element is
$ \sum_{i_n \in [d_n]} B_{i_{n^{\prime}}, i_n} \mathcal{A}_{i_1, \cdots, i_n, \cdots, i_N}$.
    \label{def3}
\end{definition}




\begin{definition}[Tucker decomposition]
For the $N$-order tensor $\mathcal{A}$, let $r_n$ and $U_n$ be the rank and left singular matrix of $\mathcal{M}_n(\mathcal{A})$, respectively. Then corresponding Tucker decomposition is given by
$$
\mathcal{A} = \mathcal{G} \times_1 U_1 \times_2 U_2 \times_3 \cdots \times_N U_N =: \mathcal{G} \times_{n \in [N]} U_n,
$$
where $\mathcal{G} \in \mathbb{R}^{r_1 \times r_2 \times \cdots \times r_N}$ is called the core tensor, and $\left(r_1, \cdots, r_N\right)$ is referred to as the multi-linear rank of the tensor $\mathcal{A}$.
\label{def5}
\end{definition}

Our general theory will involve some widely accepted concepts from the high-dimensional statistics literature. To this end, we provide the following relevant definitions \citep{raskutti2019convex}.

\begin{definition}[Weak decomposability of norm]
For a given pair of linear subspaces $ (\mathbb{A}, \mathbb{B})$, where $ \mathbb{B} \subseteq \mathbb{A} $. If the norm $ R(\cdot) $ satisfies the condition
$$ \forall \mathcal{A} \in \mathbb{A}^\perp, \, \mathcal{B} \in \mathbb{B}, \quad R(\mathcal{A}+\mathcal{B}) \geq R(\mathcal{A}) + c_R R(\mathcal{B}), $$
where $\mathbb{A}^\perp=\left\{\mathcal{A} \in \mathbb{A}^\perp | \langle \mathcal{A},\mathcal{C} \rangle =0,  \forall \mathcal{C} \in \mathbb{A}\right\}$ and $0 < c_R \leq 1$, then this norm is called a weakly decomposable norm.
\label{def6}
\end{definition}

\begin{definition}[Compatibility constant]
For a
subspace $\mathbb{A} \subset \mathbb{R}^{d_1 \times \cdots \times d_N}$, the compatibility constant 
$\phi$ is defined as $$\phi:=\sup_{\mathcal{A} \in \mathbb{A}\setminus \{0\}} \frac{R^2(\mathcal{A})}{\|\mathcal{A}\|_F^2}.$$
\label{def7}
\end{definition}

\begin{definition}[Gaussian width]
For a
set $\mathbb{A}$, the Gaussian width 
$w(\mathbb{A})$ is defined as $$w(\mathbb{A}):=\mathbb{E} \left(\sup_{\mathcal{A} \in \mathbb{A}} \langle \mathcal{A}, \mathcal{G}\rangle\right),$$ where $\mathcal{G}$ is an $N$-th order tensor and $\operatorname{vec}\left(\mathcal{G} \right) \sim \mathcal{N}(0, I)$.

\label{def8}
\end{definition}

\section{Problem Setting}
This work considers the following tensor bandits problems. At each decision step $t \in [T]$, the agent can access the arm context set $\mathbb{X}_t$, where the elements are tensors of dimensions $\left(d_1, d_2, \cdots, d_N\right)$. The agent then selects a corresponding action from the action set based on the context and will receive a corresponding reward of
$p\left(y_t \mid \mathcal{X}_t,\Theta^* \right)=\exp \left(\frac{y_t \langle \mathcal{X}_t,\Theta^* \rangle-b\left(\langle \mathcal{X}_t,\Theta^* \rangle \right)}{\phi}+c\left(y_t, \phi\right)\right)$, and $\mathbb{E}\left(y_t \mid \mathcal{X}_t, \Theta^* \right)=b^{\prime}\left(\left\langle \mathcal{X}_t, \Theta^*\right\rangle\right) := \mu\left(\left\langle \mathcal{X}_t, \Theta^*\right\rangle\right)$,
where $\Theta^* \in \mathbb{R}^{d_1 \times d_2 \times \cdots \times d_N}$ is an unknown parameters, $\mu(\cdot)$ is the inverse link function. The above equation can be rewritten as
$$
y_t=\mu \left(\left\langle\mathcal{X}_t, \Theta^* \right\rangle\right)+\varepsilon_t,
$$
where $\varepsilon_t$ is an independent $R$-sub-Gaussian noise.

The agent's goal is to minimize the expected cumulative regret relative to the optimal action $a_t^*:=\arg \max _{\mathcal{X}_{a_t} \in \mathbb{X}_t}\langle\mathcal{X}_{a_t}, \Theta^* \rangle$  over the total number of rounds $T$: 
$$
\mathbb{E}(R_T):=\mathbb{E}\left[\sum_{t \in[T]}\mu\left(\left\langle\mathcal{X}_{a_t^*}, \Theta^*\right\rangle\right)-\mu\left(\left\langle\mathcal{X}_t, \Theta^* \right\rangle\right)\right].
$$

To facilitate this study, we have established a set of widely adopted assumptions regarding the distribution of context, bounded norms, and the generalized linear model. These assumptions form the fundamental basis for our subsequent theoretical analysis. Additionally, these assumptions are common in the relevant literature \cite{lu2021low,qin2023stochastic,yi2024effective}, are not overly strict, and have broad applicability.

\begin{assumption}[Context distribution] \label{as1}
    Let $\boldsymbol{x}_{i}=\operatorname{vec}\left(\mathcal{X}_{i}\right)\in \mathbb{R}^{ d_1 d_2 \cdots d_N}$, i.e., the vectorized covariate derived from the $i$-th context tensor. Then $\boldsymbol{x}_i$ is independent and follows the $k$-sub-Gaussian distribution with covariance matrix $\Sigma$, where $ \lambda_{\min }(\Sigma) \geq c_{\ell}^2=\frac{1}{d_1 d_2\cdots d_N}$, $k=\frac{1}{\sqrt{d_1 d_2\cdots d_N}}$.

\end{assumption}

\begin{assumption}[Bounded norm]\label{as2}
The true parameter and context tensors have bounded norms, specifically $\|\Theta^*\|_F \leq 1$ and $\|\mathcal{X}\|_F \leq 1, \forall \mathcal{X} \in \mathbb{X}_t$.
\end{assumption}

\begin{remark}
     It should be clarified that Assumption \ref{as1} and Assumption \ref{as2} can hold simultaneously, for instance, when the context tensor is generated uniformly from the unit sphere.
\end{remark}

\begin{assumption}[Inverse link functions]\label{as3}
The inverse link function $\mu(\cdot)$ has the first bounded derivative, which satisfies $0 < |\mu^{\prime}(x)| \leq k_\mu, \forall |x| \leq 1$.

\end{assumption}

\begin{remark}
     The bounded constants in Assumption \ref{as3} are commonly present in the distributions of generalized linear models. For example, in the binary logistic model, $\mu(x) = \frac{1}{1 + e^{-x}}$, where $k_\mu = \frac{1}{4}$; in the Poisson model, $\mu(x) = e^x$, where $k_\mu = e$.
     
\end{remark}

\section{Main Results}

In this section, we will propose a tensor algorithm for high-dimensional settings that addresses general low-dimensional structures and discuss its regret bounds. 

\subsection{The Unified Algorithm}

\begin{algorithm}[t]
\caption{Generalized Explore Low-dimensional structure Then Commit (G-ELTC)} 
\hspace*{0.02in} {\bf Input:} 
parameters $\lambda, T, R(\Theta), c$. 
\begin{algorithmic}[1]
\For{$t=1$ to $T_1=c \phi w^2(\boldsymbol{\Theta})$} 
\State Observe $K$ contexts, $\mathcal{X}_{ 1}, \mathcal{X}_{ 2}, \cdots, \mathcal{X}_{ K}$. 
\State Choose action $a_t$ uniformly randomly, and receive reward $y_t$.
\EndFor
\State Compute the estimator $\hat{\Theta}_{T_1}$ by minimizing the problem as stated in \eqref{eqn1}.
\For{$t=1$ to $T_2=T-T_1$} 
\State Take actions $i_t=\arg \max_{i } \mu\left(\langle \mathcal{X}_i, \hat{\Theta}_{T_1} \rangle\right)$.
\EndFor

\end{algorithmic}
\label{algo1}
\end{algorithm}


Before designing the specific algorithm, it is important to clarify that a key issue in the decision-making strategy is that the true parameters are unknown during the decision process. Therefore, we first address the task of estimating unknown parameters in the decision process, which forms the foundation for our learning strategy. To this end, we provide parameter estimation results under a generalized linear model with general low-dimensional structure constraints.

Based on the generalized linear relationship, we consider that $\hat{\Theta}_{T}$ is obtained by minimizing the regularization problem with norm penalty: 
\begin{equation}
    \left\{ 
\frac{1}{T} \sum_{t \in [T]} [b(\langle \mathcal{X}_t,\Theta \rangle)-y_t \langle \mathcal{X}_t,\Theta \rangle]+\lambda_T R(\Theta)\right\}, \label{eqn1} 
\end{equation}
where $\lambda_T$ is a tuning parameter with theoretical guidance. The regularization term $R(\Theta)$ here adopts different weakly separable norms tailored to the specific low-dimensional structures involved.


    

Next, we present a general error bound for the above parameter estimation problem, which is an extension of the results by \cite{raskutti2019convex}.

\begin{theorem}[Error for parameter estimation] \label{theo1}
If $\lambda_T \geq \alpha R\left(\frac{1}{T} \sum_{t \in [T]} \epsilon_t \mathcal{X}_t\right)$, and Assumption \ref{as1}-\ref{as3} hold,
then for any $T \geq c \phi w^2(\boldsymbol{\Theta})$ such that with probability at least $1 -\delta$,
\begin{align*}
    & \max \left\{\left\|\widehat{\Theta}_{T}-\Theta^* \right\|_{T}^2,\left\|\widehat{\Theta}_{T}-\Theta^* \right\|_{\mathrm{F}}^2\right\} \\  \leq &\frac{36(1+c_R)^2 \phi \lambda^2}{(3+c_R)^2c_l^2 k_u^2},
\end{align*}
where $\alpha=\frac{c_R+3}{2c_R}$, $\boldsymbol{\Theta}=\{\Theta| R(\Theta)\leq 1\}$, and $\|\cdot\|_T^2=\frac{1}{T}\sum_{i \in [T]} \langle  \cdot, \mathcal{X}_i \rangle^2$ is the empirical norm. 
\end{theorem}

\begin{remark}
Compared to the results in \cite{raskutti2019convex}, our paper differs in its problem setup by extending the linear model to a generalized linear model and expanding the distribution of $\mathcal{X}$ from Gaussian to sub-Gaussian. Consequently, although the basic proof framework remains the same, the specific concentration inequalities employed are significantly different and more complex.
\end{remark}


We now propose an algorithm for the tensor bandits problems in high-dimensional setting, called the Generalized Explore Low-dimensional structure Then Commit (G-ELTC) algorithm. The algorithm consists of two stages: the first stage involves randomly exploring the arms, and the second stage focuses on selecting and submitting the best arm. Algorithm \ref{algo1} provides a detailed description of this process.


The reason for adopting the aforementioned Explore-then-Commit framework is that, in the case of sparse vectors, \cite{hao2020high} proposed the "Explore-the-Sparsity-Then-Commit" algorithm based on this framework, specifically targeting data-poor regions (including high-dimensional settings). Subsequently, \cite{jang2022popart} further improved this algorithm and derived a regret bound with the optimal order, where the order with respect to $T$ is $2/3$. Additionally, related research \cite{jang2024efficient} has applied this framework in the case of low-rank matrices, achieving a lower bound in terms of the dimensions. In high-dimensional settings, the ideal theoretical result is to obtain the improved bound with respect to the dimension, even if some sacrifice in the order of the number of rounds is necessary. Therefore, we also adopt this framework in the high-dimensional setting of tensors.


Compared to similar algorithms in recent papers, our algorithm has the following advantages: 
\begin{itemize}
    \item It is applicable to high-order high-dimensional settings, whereas previous studies either addressed low-order high-dimensional problems or only involved high-order extensions. 
    \item It provides a theoretical guide for the number of rounds in random exploration, which is related to the corresponding low-dimensional structure, and unifies the measurement through Gaussian width and compatibility constant. 
    \item It can be applied to a broader range of low-dimensional structural problems, with low-dimensional structures previously considered in high-dimensional settings included in this study. Furthermore, we also explore other novel low-dimensional structures, such as slice-wise sparsity and low-rankness.
\end{itemize}

\subsection{Regret Analysis}


In this section, we present the cumulative regret bounds of our general framework, along with results for specific low-dimensional tensor structures. For these examples, we first clarify all unspecified symbols in the framework, such as the parameters of weakly decomposable norms and the compatibility constants. We then provide a lemma to guide the appropriate selection of the regularization parameter $\lambda_{T_1}$, leading to the final regret bounds for the corresponding algorithms.

\begin{theorem}[Structure Independent Regret]\label{theo2}
    If $\lambda_{T_1} \geq \alpha R^*\left(\frac{1}{T_1} \sum_{t \in [T_1]}\epsilon_t \mathcal{X}_t\right)$, and Assumption \ref{as1}-\ref{as3} hold, then the expected cumulative regret of Algorithm \ref{algo1} is 
    $$\mathbb{E}(R_T) \leq 2k_\mu T_1+12k_\mu T \frac{(1+c_R)\sqrt{ \phi} \lambda_{T_1}}{(3+c_R)c_l k_u},$$
with
probability at least $1-\delta$.
\end{theorem}

\subsubsection{Tensor-Wise Low-Rankness} In scenarios like precision medicine recommendation, doctors select the best treatment for patients with specific diseases. The arms involved are drug $\times$ target $\times$ disease. In this application, the drug's therapeutic effect is typically reflected in how it interacts with the target to address the disease. In this case, tensor-wise low-rankness is the most reasonable choice because it captures the global relationship between drugs, targets, and diseases, enhancing the efficiency of recommendation.

Based on this, we discuss the theory to the low-rank tensor bandits problem. In this case, we adopt the low-rankness based on Tucker decomposition, where the rank of $\Theta^*$ is $(r_1, r_2,\cdots, r_N)$, and let $\max\{r_1, r_2,\cdots, r_N\} = r $. Consequently, the regularization norm is specified as $R(\Theta) = \|\Theta\|_{*}$. Note that in this case, $c_R = 1/2$ and $ \phi = 2r$. The following lemma provides a good choice of $\lambda_{T_1}$ for the low-rank tensor bandits problem.

\begin{lemma}\label{lem1}
    For any $\delta \in(0,1)$, let $d=\max\{d_1,d_2,\cdots,d_N\}$, let $\alpha=\frac{c_R+3}{2c_R}$, use
$$
\lambda_{T_1}=\frac{ \alpha R N}{\sqrt{T_1}}  \sqrt{2\log \frac{4T_1 N}{\delta} \log  \frac{2N(d+d^{N-1})}{\delta}}
$$
in Algorithm \ref{algo1} with $R(\Theta)=\|\Theta\|_{*}$, then with probability at least $1-\delta$, we have $\lambda_{T_1} \geq \alpha R^*\left(\frac{1}{T_1} \sum_{t \in [T_1]} \epsilon_t \mathcal{X}_t\right)$.
\end{lemma}

Given the above lemma, applying Theorem \ref{theo2}, we can derive the regret bound in the low-rank tensor setting.

\begin{corollary}\label{coro1}
Under Assumption \ref{as1}-\ref{as3}, let
$
\lambda_{T_1}=\frac{ \alpha R N}{\sqrt{T_1}}  \sqrt{2\log \frac{4T_1 N}{\delta} \log  \frac{2N(d+d^{N-1})}{\delta}}
$, then the expected cumulative regret of the Algorithm \ref{algo1} in the low multi-linear rank tensor bandits problem is upper bounded by
$$
\mathbb{E}(R_T)=\tilde{O}\left(d^\frac{N}{3} r^{\frac{1}{3}} T^{\frac{2}{3}}\right).
$$
\end{corollary}

To better illustrate the advantages of our work in the case of high-dimensional tensors with low multi-linear rank, we compared with the existing high-order tensor bandit method \cite{shi2023high} which is $\tilde{O}\left(d^{2} r^{N-2} T^{\frac{1}{2}}\right)$, our results show significant advantages in $d$ and $r$, although they are slightly higher in order for $T$. In certain high-dimensional scenarios where $r = O(1)$ and $N \leq 5$, with $d \gg T$, our bound remains competitive. Moreover, our work extends to generalized linear models, offering broader applicability. Secondly, compared to the matricization methods G-ESTT with regret $\tilde{O}\left( M^{\frac{1}{2}} d^{\frac{N-1}{2}}  r^{\frac{1}{2}} T^{\frac{1}{2}} \right)$ \cite{kang2022efficient} and LPA-ETC with regret $\tilde{O}\left( d^{N-1} r^{\frac{2}{3}}  T^{\frac{2}{3}}\right)$ \cite{jang2024efficient}, our algorithm still performs significantly better in terms of the orders of $d$ and $r$. This advantage demonstrates that our algorithm has stronger competitiveness in many high-dimensional scenarios and further highlights the crucial role of tensor representation.

\subsubsection{Slice-Wise Sparsity and Low-Rankness}

In applications like social media analysis, the platform pushes topics to specific user groups at certain times. The involved arms are time $\times$ user group $\times$ topic, where the data shows that users only participate in topic discussions at specific time points, and these users exhibit group characteristics. Therefore, slice-wise sparsity and low-rankness can more effectively capture the sparsity of users across time and reveal the potential relationships between user groups and topics. 

Next, we discuss the regret of tensor bandits problem, where the sparsity and low-rankness are at the slice level. In this scenario, take $(1, 2)$-slices of third-order tensor as an illustration. That is, it refers to the $s$ nonzero slices $\Theta^*_{\cdot \cdot k}, k \in [d_3]$ are rank-$ r$ matrices, so we choose the $R(\Theta)=\|\Theta\|_{(1,2),*}=\sum_{k \in [d_3]} \|\Theta_{\cdot  \cdot k}\|_{*} $. In this case, $c_R = 1$ and $\phi = 2r$. Similarly, the theoretical guidance for adjusting the parameter is provided by the following lemma.

\begin{lemma}\label{lem2}
    For any $\delta \in(0,1)$, let $d=\max\{d_1,d_2,d_3\}$, and $\alpha=\frac{c_R+3}{2c_R}$, use

$$
\lambda_{T_1}=\frac{c \alpha R}{\sqrt{T_1}}  \sqrt{\log \frac{4T_1d}{\delta} \log  \frac{4d}{\delta}}
$$
in Algorithm \ref{algo1} with $R(\Theta)=\|\Theta\|_{(1,2),*}$, then with probability at least $1-\delta$, we have $\lambda_{T_1} \geq \alpha R^*\left(\frac{1}{T_1} \sum_{t \in [T_1]} \epsilon_t \mathcal{X}_t\right)$.
\end{lemma}

Similarly, we can derive the regret bound in the slice-wise sparse and low-rank tensor setting.

\begin{corollary}\label{coro2}
Under Assumption \ref{as1}-\ref{as3}, let
$
\lambda_{T_1}=\frac{c \alpha R}{\sqrt{T_1}}  \sqrt{\log \frac{4T_1d}{\delta} \log  \frac{4d}{\delta}}
$, the expected cumulative regret of the Algorithm \ref{algo1} in the tensor bandits with slice-wise sparsity and low-rankness is upper bounded by
$$
\mathbb{E}(R_T)=\tilde{O}\left(d r^{\frac{1}{3}} T^{\frac{2}{3}}\right).
$$
\end{corollary}

This structurally constrained bandits problem has not been addressed in previous work. To demonstrate the superiority of our results, we compare them with the results related to tensor matricization. In \cite{kang2022efficient}, the bound is $M^{\frac{1}{2}}d r T^{1/2}$ (Note that this bound involves $M$, which is related to the dimension d), while our work improves the dependence on dimensionality and rank, although they are slightly higher in order for $T$. In certain scenarios, our bound remains competitive. Compared with \cite{jang2024efficient}, which provides a bound $\tilde{O}\left( d^2 r^{\frac{2}{3}}  T^{\frac{2}{3}}\right)$, our work further reduces the order of dimension from $2$ to $1$ and the order of rank from $2/3$ to $1/3$. These further illustrate how the process of matricization disrupts the original structure of tensors, underscoring the significance of studying slice-wise sparsity and low-rankness within tensors. 
\begin{figure}[t]
    \centering
    \subfigure[Cumulative Regret $R_T$]{
        \label{Fig.sub.1}
        \includegraphics[width=0.20\textwidth]{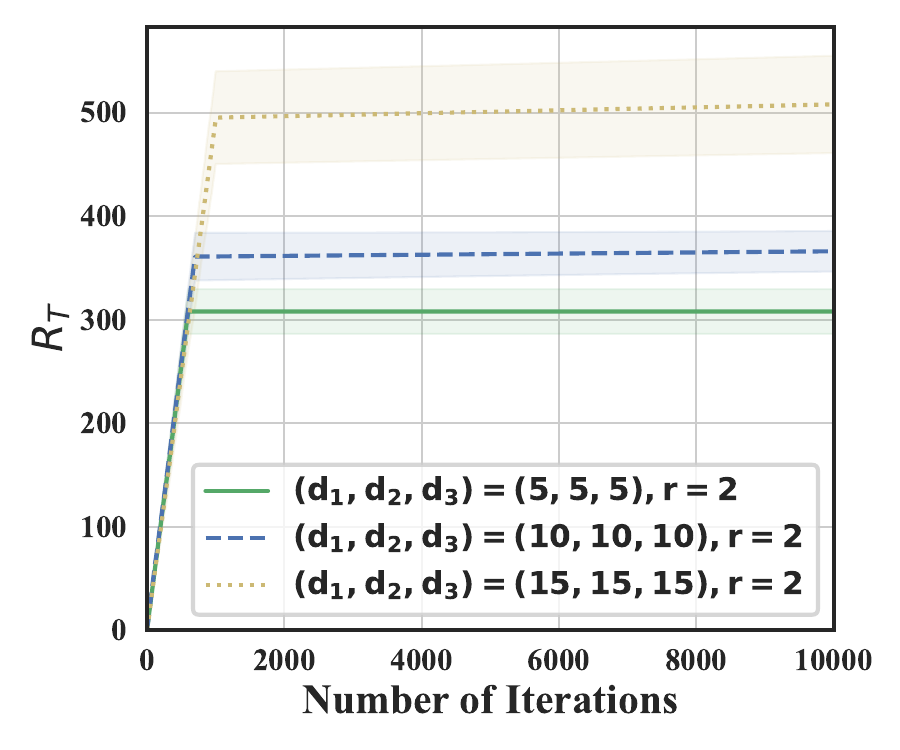}}
        \hfil
    \subfigure[The ratio $R_T/B_T$]{
        \label{Fig.sub.2}
        \includegraphics[width=0.20\textwidth]{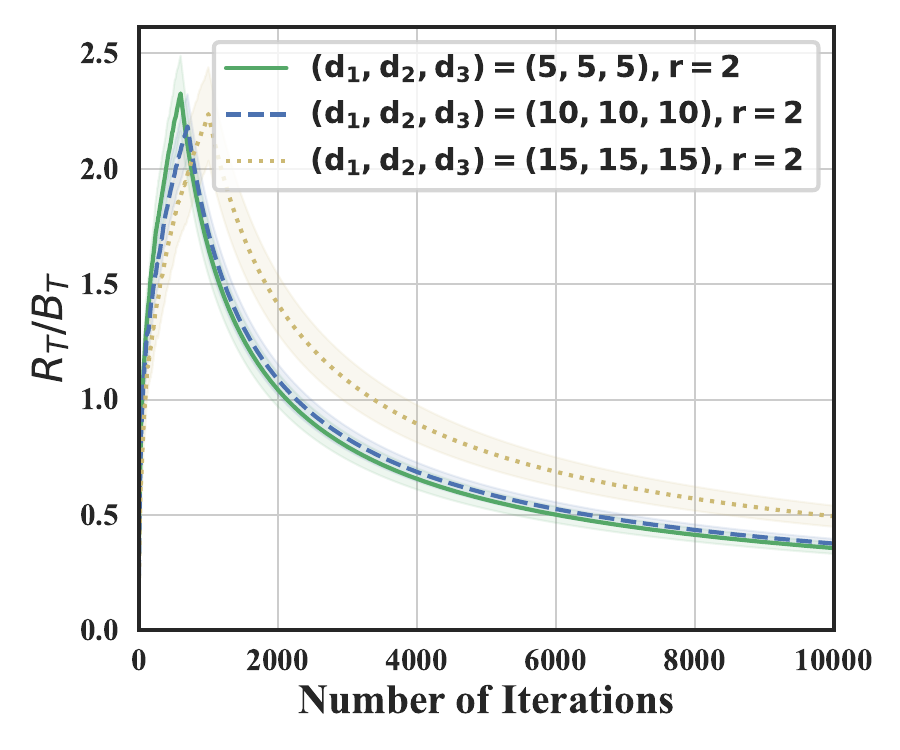}}
    \caption{Experimental results of tensor bandits under low multi-linear rankness with different dimensional settings. (a) displays the curve of cumulative regret over time, while (b) shows the variation of the ratio of cumulative regret to the theoretical bound $B_T$ over time.}
    \label{fig:figure1}
\end{figure}

\begin{figure}[t]
    \centering
    \subfigure[Cumulative Regret $R_T$]{
        \label{Fig.sub.3}
        \includegraphics[width=0.20\textwidth]{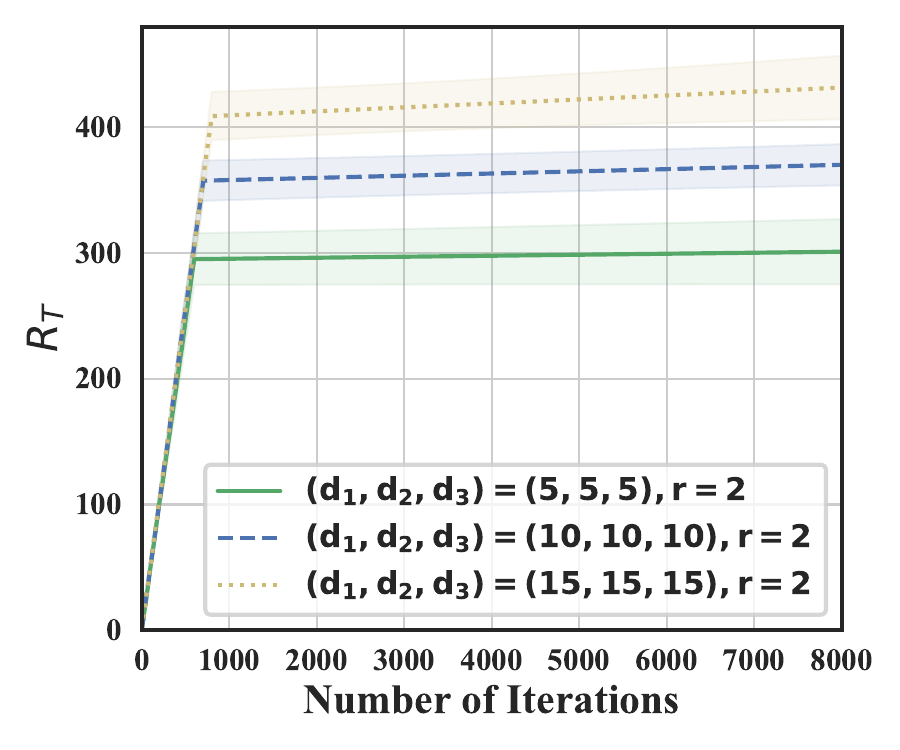}} 
        \hfil
    \subfigure[The ratio $R_T/B_T$]{
        \label{Fig.sub.4}
        \includegraphics[width=0.20\textwidth]{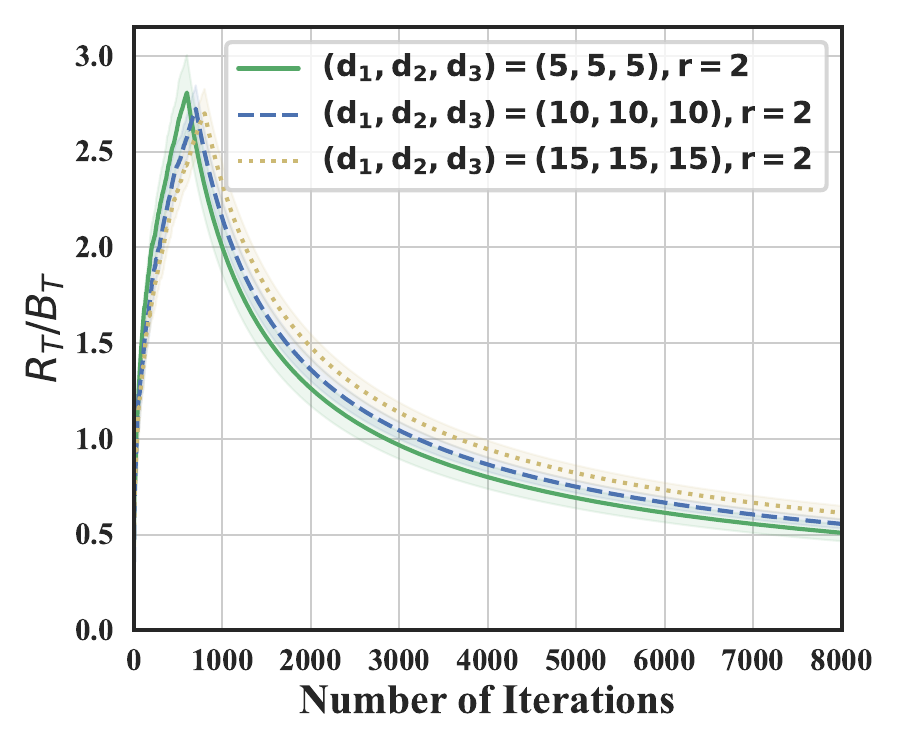}}
    \caption{Experimental results of tensor bandits under slice sparse and low rank structure with different dimensional settings. (a) displays the curve of cumulative regret over time, while (b) shows the variation of the ratio of cumulative regret to the theoretical bound $B_T$ over time.}
    \label{fig:figure2}
\end{figure}

\subsubsection{Extension to Other Low-Dimensional Structures}
In some application scenarios, sparsity may occur at the element level and the fiber level. For example, in intelligent question-answering systems, the platform matches users' questions with corresponding answers. The involved arms are user $\times$ question $\times$ answer, where only a few of the users' questions have matching answers. Therefore, entry-wise sparsity is more appropriate, as it allows optimization for specific user-question-answer pairs without assuming a global question-answer association pattern. Similarly, in mobile network resource allocation, base stations allocate resources to users at specific time periods, where only a few active users require prioritized resource allocation. Therefore, fiber-wise sparsity is more suitable for handling the sparsity of users across specific time and base station. Based on the above practical scenarios, we extend the relevant results of this paper to other low-dimensional structures, such as entry-wise sparsity and fiber-wise sparsity. The detailed theoretical results and proof of this part can be found in the the Appendix.

\section{Experiments}



In this section, we validate the effectiveness of our theoretical results through experiments. Specifically, the true parameter tensor with low multi-linear rankness $\Theta^* \in \mathbb{R}^{d_1 \times d_2 \times d_3}$ is generated as follows: we first generate the tensor with dimensions $d_1 \times d_2 \times d_3$, with its elements drawn from a standard normal distribution. We then project it into a low multi-linear rank space. While for the true parameter tensor with slice sparsity and low-rankness, it is generated as follows: first, we randomly select the indices of the non-zero slices, then the non-zero slices are randomly generated as low-rank matrices. The elements of the contextual tensor are sampled from a standard normal distribution. We construct a binary reward model, that is $y_t \sim \operatorname{Bernoulli}(\mu\left(\left \langle\mathcal{X}_t, \Theta^* \right \rangle \right))$, where $\mu(x)=\frac{e^x}{1+e^x}$. For each setting in each problem, we run the algorithm 10 independent times to compute the average results with one standard deviation error bar. 

For the two aforementioned low-dimensional structures, we conducted experiments from two main perspectives. First, we verified that under these structures, the upper bound of the cumulative regret grows sublinearly. To demonstrate this, we plotted the curve of the cumulative regret upper bound against the number of decision rounds under different dimensions (as shown in Figure \ref{Fig.sub.1} and Figure \ref{Fig.sub.3}). These figures show that in all three settings, the upper bound of cumulative regret increases with the number of random explorations and flattens with the growth of exploitation rounds, ultimately demonstrating ideal sublinear growth. Second, we confirmed that the algorithm achieves the theoretical regret bound under these structures. To demonstrate this, we plotted the ratio of the cumulative regret upper bound to the theoretical bound against the number of decision rounds (for Figure \ref{Fig.sub.2}, $B_T:=d r^{1/3}T^{2/3}$; for Figure \ref{Fig.sub.4}, $B_T:=d r^{1/3}T^{2/3}$). As shown in Figure \ref{Fig.sub.2} and Figure \ref{Fig.sub.4}, in each setting, the ratio stabilizes around a constant less than 1 as the number of rounds increases, indicating that our algorithm achieves the theoretical bound.

\section{Conclution}
This study, driven by the need for flexible decision-making strategies in high-dimensional and high-order scenarios of real-world applications, delves into the high-dimensional generalized tensor bandits problems, encompassing low-dimensional structural constraints. We employ weakly decomposable norms as convex regularizers to describe the related structures and propose a unified algorithm, along with a theoretical analysis of its cumulative regret bounds. Specifically, for tensor low-rankness structure, we achieve improved bounds in terms of both dimension and rank, effectively addressing the challenges posed by high-dimensional problems in high-order settings while extending the algorithm's applicability. Additionally, we consider the presence of several popular low-dimensional structures, such as sparsity and low-rankness, at different levels of the tensor and demonstrate how these results can be presented using different specific norms, further highlighting the flexibility and wide applicability of the proposed method.

\bibliographystyle{IEEEtran}

\bibliography{references}

\clearpage
\appendices

\section{Related Works}

The improvement in decision-making performance of contextual bandits stems from the effective utilization of feature information in the decision process. Linear bandits \cite{abbasi2011improved,chu2011contextual}, as one of the classic approaches, model the reward as the inner product of the feature vector and unknown parameters to characterize the problem. However, considering the complexity of rewards in real-world decision-making scenarios, researchers \cite{filippi2010parametric,li2017provably} have further extended this approach to generalized linear models to better capture the diversity of reward behaviors. However, these methods still struggle to effectively address the increasing complexity of decision-making scenarios, particularly as the number of interrelated factors that need to be considered grows.

To tackle this challenge, previous studies have extended the bandit problem to higher-order settings. Specifically, low-rank matrix contextual bandit models \citep{jun2019bilinear, jang2021improved, lu2021low, kang2022efficient} have been proposed to handle more complex structures. However, these studies have limited applicability in high-dimensional environments, as their cumulative regret bounds depend on higher-order of dimensions. To resolve this issue, \cite{jang2024efficient} proposed a lower bound regarding the dimension using the ETC framework. Compared to these studies, the results in this paper are based on the low multi-linear rank assumption, which not only encompasses the existing works but also extends to higher-order settings. Furthermore, although this paper shares similarities with the study by \cite{li2022simple}, the low-dimensional structures involved in this paper are more complex.

In higher-order scenarios, the low-rank tensor contextual bandits problem arises. Unlike low-rank matrices, tensor decomposition is not unique, and the rank of a tensor can be defined in various ways. For Tucker decomposition, \cite{zhou2024stochastic} was the first to focus on this issue, but it lacked theoretical guarantees. Subsequently, \cite{shi2023high} further focused on the effective utilization of low-dimensional subspaces, introducing the OFUL algorithm based on a linear parameter model and providing an upper bound on cumulative regret. For other common decomposition methods, namely CP and T-SVD, \cite{ide2022targeted} and \cite{yi2024effective} derived relevant results, respectively. However, it is noteworthy that these studies did not address the high-dimensional challenges faced by low-rank tensor contextual bandits. Compared to T-SVD and CP, Tucker decomposition offers greater flexibility and expressive power when handling high-order data \cite{kolda2009tensor}. Therefore, this paper adopts Tucker decomposition while considering low-rank tensor structures and achieves lower-order regret bounds concerning the dimension for low-rank tensors, effectively addressing the high-dimensional challenges in higher-order settings. Additionally, this paper also considers other low-dimensional structural constraints in tensors, such as slice sparsity and low-rankness.

\section{Auxiliary Lemmas}
Before proceeding with the formal proof, we need first to state the lemmas that will be used.
\begin{lemma}[Theorem D  \citep{mendelson2007reconstruction} ]\label{ref.lem1}
 There exist absolute constants $c_1, c_2, c_3$ for which the following holds. Let $(\Omega, \mu)$ be a probability space, set $\mathbb{F}$ be a subset of the unit sphere of $L_2(\mu)$, i.e., $\mathbb{F} \subseteq S_{L_2}=\left\{f|\|f\|_{L_2}=1\right\}$, and assume that $\operatorname{diam}\left(\mathbb{F},\|\cdot\|_{\psi_2}\right) \leq \kappa$. Then, for any $\theta>0$ and $n \geq 1$ satisfying
$$
c_1 \kappa \gamma_2\left(\mathbb{F},\|\cdot\|_{\psi_2}\right) \leq \theta \sqrt{n},
$$
with probability at least $1-\exp \left(-c_2 \theta^2 n / \kappa^4\right)$,
$$
\sup _{f \in \mathbb{F}}\left|\frac{1}{n} \sum_{i \in [n]} f^2\left(\mathcal{X}_i\right)-\mathbb{E}\left(f^2\right)\right| \leq \theta .
$$
Further, if $\mathbb{F}$ is symmetric, then
$$
E\left[\sup _{f \in \mathbb{F}}\left|\frac{1}{n} \sum_{i\in [n]} f^2\left(\mathcal{X}_i\right)-\mathbb{E}\left(f^2\right)\right|\right] \leq  c_3 \kappa^2 \frac{\gamma_2\left(\mathbb{F},\| \cdot\|_{\psi_2}\right)}{\sqrt{n}}.
$$
\end{lemma}






\begin{lemma}[Talagrand’s majorizing measure theorem \citep{vershynin2018high}]\label{ref.lem2}
Let $\{X_t\}_{t\in \mathbb{T}}$ be a mean zero Gaussian process on a set $\mathbb{T}$. Consider the canonical metric defined on $\mathbb{T}$, i.e. $d(t, s) = \|X_t-X_s\|_{L_2}$. Then
    $$c
\gamma_2\left(\mathbb{T},d\right) \leq w(\mathbb{T}) \leq C \gamma_2\left(\mathbb{T},d\right).
$$
\end{lemma}
\begin{lemma}[Tail bound \citep{vershynin2018high}] \label{ref.lem3}
Let $\{X_x\}_{x \in \mathbb{X}}$ be a random process on a subset $\mathbb{X} \subset \mathbb{R}^d$. If $\|X_x-X_y\|_{\psi_2} \leq L \|x-y\|_2$, then for every $u \geq 0$ we have
$$
\sup _{x \in \mathbb{X}}\left|X_x\right| \leq C L\left(w(\mathbb{X})+u  \operatorname{rad}(\mathbb{X})\right),
$$
with probability at least $1-2 \exp \left(-u^2\right)$, where $\operatorname{rad}(\mathbb{X}):=\sup_{x \in \mathbb{X}}\|x\|_2$ is the radius, $w(\mathbb{X}):=\mathbb{E} \left(\sup_{x \in \mathbb{X}} \langle x, g \rangle\right), g\sim \mathcal{N}(0, I)$ is the Gaussian width.
\end{lemma}


\begin{lemma}[Operations on sub-Gaussian variables \citep{wainwright2019high}]\label{ref.lem5}
     Suppose that $x_{1}$ and $x_{2}$ are zero-
mean and sub-Gaussian with parameters $\sigma_1$ and $\sigma_2$, respectively.  

\begin{itemize}
    \item If $x_1$ and $x_2$ are independent, then the random variable $x_1+x_2$ is sub-Gaussian
with parameter $\sqrt{\sigma_{1}^{2}+\sigma_{2}^{2}}.$
\item In general (without assuming independence), the random variable $x_1+x_2$ is
sub-Gaussian with parameter at most $\sqrt{2}\sqrt{\sigma_{1}^{2}+\sigma_{2}^{2}}.$
\end{itemize}
\end{lemma}

\begin{lemma}[Norm Inequalities]\label{ref.lem6}
For $\boldsymbol{x} \in \mathbb{R}^d$, we have $\begin{cases} 
\|\boldsymbol{x}\|_p \leq \|\boldsymbol{x}\|_q & \text{if } 1 \leq q \leq p , \\ 
\|\boldsymbol{x}\|_p \leq d^{\frac{1}{p}-\frac{1}{q}} \|\boldsymbol{x}\|_q & \text{if } 1 \leq p \leq q .
\end{cases}$
\end{lemma}

\begin{proof}
    For $1 \leq q \leq p $, 
    \begin{align}
        \|\boldsymbol{x}\|_p^p=&\sum_{i \in [d]}|\boldsymbol{x}_i|^q |\boldsymbol{x}_i|^{p-q}\\
        \leq & \sum_{i \in [d]}|\boldsymbol{x}_i|^q \sum_{i \in [d]}|\boldsymbol{x}_i|^{p-q}\\
        = &\sum_{i \in [d]}|\boldsymbol{x}_i|^q \sum_{i \in [d]}\left(|\boldsymbol{x}_i|^q\right)^{\frac{p-q}{q}}\\
        \leq & \sum_{i \in [d]}|\boldsymbol{x}_i|^q \left(\sum_{i \in [d]}|\boldsymbol{x}_i|^q\right)^{\frac{p-q}{q}}\\
        =&\left(\sum_{i \in [d]}|\boldsymbol{x}_i|^q\right)^{\frac{p}{q}},
    \end{align} that is $\|\boldsymbol{x}\|_p \leq \|\boldsymbol{x}\|_q$. For $1 \leq p \leq q$, 
    \begin{align}
        \|\boldsymbol{x}\|_p^p=&\sum_{i \in [d]}|\boldsymbol{x}_i|^p \\
        \leq & \left(\sum_{i \in [d]}\left(|\boldsymbol{x}_i|^p\right)^{\frac{q}{p}}\right)^{\frac{p}{q}} \left(\sum_{i \in [d]}1^{\frac{q}{q-p}}\right)^{\frac{q-p}{q}}\\
        =& \left(\sum_{i \in [d]}|\boldsymbol{x}_i|^q\right)^{\frac{p}{q}} d^{\frac{q-p}{q}},
    \end{align}
    that is $\|\boldsymbol{x}\|_p \leq d^{\frac{1}{p}-\frac{1}{q}} \|\boldsymbol{x}\|_q$.
\end{proof}

\section{Proof of Theorem \ref{theo1}}

\begin{proof}
    Let $L(\Theta)=\frac{1}{T_1} \sum_{t \in [T_1]} [b(\langle \mathcal{X}_t,\Theta \rangle)-y_t \langle \mathcal{X}_t,\Theta \rangle]$, $\Delta=\hat{\Theta}-\Theta^*$. According to the optimal problem and Taylor's expansion about $b(\cdot)$, we have $L(\hat{\Theta})+\lambda_{T_1} R\left(\hat{\Theta}\right)\leq L(\Theta^*)+\lambda_{T_1}  R\left(\Theta^*\right)$ and $b(\langle \mathcal{X}_t,\hat{\Theta} \rangle)=b(\langle \mathcal{X}_t,\Theta^* \rangle)+b^{\prime} (\langle \mathcal{X}_t,\Theta^* \rangle) \langle \mathcal{X}_t,\Delta \rangle+\frac{b^{\prime \prime} (\langle \mathcal{X}_t,\Theta^*+\eta \Delta \rangle)}{2}(\langle \mathcal{X}_t,\Delta \rangle)^2.$

By rearranging and organizing the above results, we have 
\begin{align}
    &\frac{1}{2T_1} \sum_{t \in [T_1]} [\mu^{\prime} (\langle \mathcal{X}_t,\Theta^*+\eta \Delta \rangle)(\langle \mathcal{X}_t,\Delta \rangle)^2 \\
    \leq & \frac{1}{T_1} \sum_{t \in [T_1]} \epsilon_t\langle \mathcal{X}_t,\Delta \rangle+\lambda_{T_1} \left[R\left(\Theta^*\right)- R\left(\hat{\Theta}\right)\right], \\
    \leq & R^*\left(\frac{1}{T_1} \sum_{t\in [T_1]} \epsilon_t X_t \right) R\left(\Delta\right)+\lambda_{T_1} \left[R\left(\Theta^*\right)- R\left(\hat{\Theta}\right)\right], \\
    \leq & R^*\left(\frac{1}{T_1} \sum_{t\in [T_1]} \epsilon_t \mathcal{X}_t \right) \left[R\left(\Delta_{\mathcal{A}}\right)+R\left(\Delta_{\mathcal{A}^\perp}\right) \right] \notag\\
    &+\lambda_{T_1} \left[-c_R R\left(\Delta_{\mathcal{A}^\perp}\right)+R\left(\Delta_{\mathcal{A}}\right)\right],\\
    =&\left[ R^*\left(\frac{1}{T_1} \sum_{t\in [T_1]} \epsilon_t \mathcal{X}_t \right)-c_R \lambda_{T_1} \right]R\left(\Delta_{\mathcal{A}^\perp}\right)\notag \\
    &+\left[ R^*\left(\frac{1}{T_1} \sum_{t\in [T_1]} \epsilon_t \mathcal{X}_t \right) +\lambda_{T_1} \right]R\left(\Delta_{\mathcal{A}}\right),\\
    \leq &-\frac{c_R+c_R^2}{3+c_R}\lambda_{T_1} R\left(\Delta_{\mathcal{A}^\perp}\right)+\frac{3+3c_R}{3+c_R}\lambda_{T_1} R\left(\Delta_{\mathcal{A}}\right),\\
    \leq & \frac{3+3c_R}{3+c_R} \lambda_{T_1} R\left(\Delta_{\mathcal{A}}\right),
\end{align}
where the second inequality comes from the Cauchy-Schwarz inequality, and the third inequality comes from the definition of the separable norm and the norm triangle inequality. According to Assumption 3, to prove the above equation, it is sufficient to show that
\begin{align}
    &\frac{k_\mu}{2T_1} \sum_{t\in [T_1]} (\langle \mathcal{X}_t,\Delta \rangle)^2 \leq \frac{3+3c_R}{3+c_R}\lambda_{T_1} R\left(\Delta_{\mathcal{A}}\right) \\
    \leq &\frac{3+3c_R}{3+c_R}\lambda_{T_1} \sqrt{\phi}\left\|\Delta_{\mathcal{A}}\right\|_F \leq \frac{3+3c_R}{3+c_R}\lambda_{T_1} \sqrt{\phi}\left\|\Delta\right\|_F,
\end{align}
that is $\left\|\Delta\right\|_{T_1}^2 \leq \frac{6(1+c_R)}{k_\mu(3+c_R)}\lambda_{T_1}\sqrt{\phi}\left\|\Delta\right\|_F:=\delta_{T_1} \left\|\Delta\right\|_F$. Noting that our goal is to prove an upper bound for $\|\Delta\|_F$, we need to utilize the relationship between $\|\Delta\|_F$ and $\|\Delta\|_{T_1}$. Specifically, we classify the cases into three settings: 
\begin{itemize}
    \item the first case is $\left\{\|\Delta\|_{T_1} \geq \|\Delta\|_F\right\}$;
    \item the second case is $\left\{\|\Delta\|_{T_1} \leq \|\Delta\|_F, \|\Delta\|_F \leq \delta_{T_1}\right\}$;
    \item the third case is $\left\{\|\Delta\|_{T_1} \leq \|\Delta\|_F, \|\Delta\|_F \geq \delta_{T_1}\right\}$.
\end{itemize}

For the first two cases, we can obtain \(\max\{\|\Delta\|_{T_1}, \|\Delta\|_F\} \leq \delta_{T_1}\) through simple calculations. However, in the third case, we need to further apply concentration inequalities for the analysis. 

We claim that it suffices to prove that there exists \( 1>a >0 \) such that
\[
\mathbb{P}\left(\sup_{\Delta \in \mathbb{A}} \|\Delta\|_{T_1}^2 \geq a\|\Delta\|_F^2\right)
\]
holds with high probability, where $$\mathbb{A} := \left\{ \Delta \in \mathbb{R}^{d_1 \times \cdots \times d_N} \,\big|\, R\left(\Delta_{\mathcal{A}^\perp}\right) \leq \frac{3}{c_R} R\left(\Delta_{\mathcal{A}}\right) \right\}.$$
We consider the following class of functions $$\mathbb{F}=\left\{\frac{\langle\cdot, \widetilde{\Delta} \rangle}{\sqrt{\operatorname{vec}^\top(\widetilde{\Delta}) \Sigma \operatorname{vec}(\widetilde{\Delta})}} | \widetilde{\Delta}=\frac{\Delta}{\|\Delta\|_F}, \Delta\in \mathbb{A}\right\}.$$ Since $\|f\|_{L_2}^2=\frac{\mathbb{E}\left\langle \mathcal{X}_i, \widetilde{\Delta} \right\rangle ^2}{\operatorname{vec}^\top(\widetilde{\Delta}) \Sigma \operatorname{vec}(\widetilde{\Delta})}=1$, then $\mathbb{F}$ is a subset of the unit sphere, i.e., $\mathbb{F} \subseteq S_{L_2}$. Further, \begin{align}
    &\operatorname{diam}\left(\mathbb{F},\|\cdot\|_{\psi_2}\right) =\sup _{f_1,f_2 \in \mathbb{F}}\left\|f_1-f_2\right\|_{\psi_2}\\
    \leq & \frac{1}{c_l}\sup _{\widetilde{\Delta}}\left\|\langle \mathcal{X}_i, \widetilde{\Delta}_1 \rangle-\langle \mathcal{X}_i, \widetilde{\Delta}_2 \rangle\right\|_{\psi_2}\\
    \leq & \left\|\operatorname{vec}(\mathcal{X}_i)\right\|_{\psi_2} \sup _{\widetilde{\Delta}} \|\widetilde{\Delta}_1-\widetilde{\Delta}_2\|_F   \leq \frac{2k}{c_l}. 
\end{align}

Then, from Lemma \ref{ref.lem1}, it follows that with probability at least $1-\exp \left(-c_2  \theta^2 T_1 c_l^4/k^4\right)$, for $
\theta \geq  \frac{c_1 k \gamma_2\left(\mathbb{F} ,\| \cdot \|_{\psi_2}\right)}{c_l\sqrt{T_1}}$,
we have
\begin{align}
   &\sup _{f \in \mathbb{F}}\left|\frac{1}{T_1} \sum_{i \in [T_1]}f^2(\mathcal{X}_i)-\mathbb{E}\left(f^2\right)\right|\\
   =&\sup _{\widetilde{\Delta} \in \widetilde{\mathbb{A}}  } \left| \frac{\|\widetilde{\Delta}\|_{T_1}^2}{\operatorname{vec}^\top(\widetilde{\Delta}) \Sigma \operatorname{vec}(\widetilde{\Delta})}-1\right| \leq \theta , 
\end{align}
where $\mathbb{E}\left(f^2\right)=\|f\|_{L_2}^2=1$,
\begin{align*}
   \widetilde{\mathbb{A}} := \left\{ \widetilde{\Delta} \in \mathbb{R}^{d_1 \times \cdots \times d_N} \,\big|\, R\left(\widetilde{\Delta}_{\mathcal{A}^\perp}\right) \leq \frac{3}{c_R} R\left(\widetilde{\Delta}_{\mathcal{A}}\right),  \|\widetilde{\Delta}\|_F=1 \right\}. 
\end{align*}

Next, based on Lemma \ref{ref.lem2}, we have 
\begin{equation}
  \gamma_2\left(\mathbb{F} ,\|\cdot\|_{\psi_2}\right) \leq \frac{k}{c_l} \gamma_2\left(\mathbb{F},\|\cdot\|_{L_2}\right) \leq \frac{k}{c_l} c_4 w(\widetilde{\mathbb{A}}).  
\end{equation}
Therefore, we set $\theta= \frac{c_1 c_4 k^2 w(\widetilde{\mathbb{A}})}{c_l^2\sqrt{T_1}} $. As a result, we have
\begin{equation}
    1-\theta \leq 
\sup _{\widetilde{\Delta} } \frac{\|\widetilde{\Delta}\|_{T_1}^2}{\operatorname{vec}^\top(\widetilde{\Delta}) \Sigma \operatorname{vec}(\widetilde{\Delta})} \leq 1+\theta.
\end{equation}
yielding 
\begin{equation}
  \lambda_{\min} (\Sigma) \|\widetilde{\Delta}\|_F^2(1-\theta) \leq \sup _{\widetilde{\Delta} } \|\widetilde{\Delta}\|_{T_1}^2 \leq \lambda_{\max} (\Sigma) \|\widetilde{\Delta}\|_F^2(1+\theta).  
\end{equation}
Then let $1>c_l^2 \|\widetilde{\Delta}\|_F^2(1-\theta)>0$, that is $$T_1\geq \frac{c_1^2 c_4^2 k^4 w^2(\mathbb{\widetilde{\mathbb{A}}})}{c_l^4},$$ we can get $
  \sup _{\widetilde{\Delta} } \|\widetilde{\Delta}\|_{T_1}^2 \geq c_l^2 \|\widetilde{\Delta}\|_F^2
$. It is equivalent to \( \sup_{\Delta} \|\Delta\|_{T_1}^2 \geq c_l^2 \|\Delta\|_F^2 \). Note that $w(\mathbb{\widetilde{\mathbb{A}}})\leq w\left(\|\cdot\|_F \leq 1\right)\leq \sqrt{\phi}w\left(R(\cdot) \leq 1\right)$, $c_l^2 \leq \|\Sigma\|_F\leq \mathbb{E}\|\mathcal{X}_i\|_F^2 \leq 1$. Therefore, we prove that \( \max\{\|\Delta\|_{T_1}, \|\Delta\|_F\} \leq \frac{\delta_{T_1}}{c_l} \).

\end{proof}

\section{Proof of Theorem \ref{theo2}}

\begin{proof}
    We adjust the scale of the instantaneous regret in the $t$-th round according to the different stages of the algorithm. In stage 1, i.e., $t \in[T_1]$, according to the Fundamental Theorem of Calculus and bounded norm assumption, we have $\left|\mu\left(\left\langle\mathcal{X}_{a_t^*}, \Theta^*\right\rangle\right)-\mu\left(\left\langle\mathcal{X}_t, \Theta^* \right\rangle\right)\right| \leq k_\mu \left|\left\langle\mathcal{X}_{a_t^*}-\mathcal{X}_t, \Theta^*\right\rangle\right|\leq k_\mu \left\|\mathcal{X}_{a_t^*}-\mathcal{X}_t\right\|_F \left\|\Theta^*\right\|_F\leq 2k_\mu$. For the second stage, considering the selection of arms, we have \(\left\langle\mathcal{X}_{a_t^*}-\mathcal{X}_t,\hat{\Theta}_{T_1}\right\rangle \leq 0\). Therefore, the following inequality holds:
    \begin{align}
        &\left|\mu\left(\left\langle\mathcal{X}_{a_t^*}, \Theta^*\right\rangle\right)-\mu\left(\left\langle\mathcal{X}_t, \Theta^* \right\rangle\right)\right|\\ 
\leq & k_\mu \left|\left\langle\mathcal{X}_{a_t^*}-\mathcal{X}_t, \Theta^*-\hat{\Theta}_{T_1}\right\rangle\right| \\
\leq & k_\mu \left\|\mathcal{X}_{a_t^*}-\mathcal{X}_t\right\|_F \left\|\Theta^*-\hat{\Theta}_{T_1}\right\|_F \\
\leq & 2k_\mu  \left\|\Theta^*-\hat{\Theta}_{T_1}\right\|_F.
    \end{align}
We can further obtain the bound on cumulative regret as follows:
\begin{align}
    \mathbb{E}(R_T)&=\mathbb{E}\left(\sum_{t \in[T]}\mu\left(\left\langle\mathcal{X}_{a_t^*}, \Theta^*\right\rangle\right)-\mu\left(\left\langle\mathcal{X}_t, \Theta^* \right\rangle\right)\right)\\
&=\mathbb{E}\left(\sum_{t \in[T_1]}\mu\left(\left\langle\mathcal{X}_{a_t^*}, \Theta^*\right\rangle\right)-\mu\left(\left\langle\mathcal{X}_t, \Theta^* \right\rangle\right)\right) \notag \\
&+\mathbb{E}\left(\sum_{t \in[T-T_1]}\mu\left(\left\langle\mathcal{X}_{a_{T_1+t}^*}, \Theta^*\right\rangle\right)-\mu\left(\left\langle\mathcal{X}_{T_1+t}, \Theta^* \right\rangle\right)\right)\\
&\leq 2k_\mu T_1+2k_\mu T \left\|\Theta^*-\hat{\Theta}_{T_1}\right\|_F\\
&\leq 2k_\mu T_1+12k_\mu T \frac{(1+c_R)\sqrt{ \phi} \lambda_{T_1}}{(3+c_R)c_l k_u}.
\end{align}

\end{proof}

\section{Proof of Lemma \ref{lem1}}

Now we consider a specific low-dimensional structure.  Specifically, we adopt low-rankness based on Tucker decomposition, where the rank of $\Theta^* \in \mathbb{R}^{d_1\times d_2 \times \cdots \times d_N}$ is $(r_1, r_2,\cdots, r_N)$, and let $\max\{r_1, r_2,\cdots, r_N\} = r$, $\max\{d_1, d_2,\cdots, d_N\} = d$. Note that in this case, the regularization norm is $R(\Theta) = \|\Theta\|_{*}=\frac{1}{N} \sum_{j \in [N]} \mathcal{M}_j(\Theta)$. Then $R^*(\Theta) = \|\Theta\|=N \max_{j \in [N]} \left\{\left\|\mathcal{M}_j(\Theta)\right\|,j \in [N] \right\}$.

\begin{proof}
    In this case, the condition that $\lambda_{T_1}$ must satisfy is
    \begin{align}
    &\mathbb{P}\left(\frac{c_R+3}{2c_R}R^*\left(\frac{1}{T_1} \sum_{t\in [T_1]} \epsilon_t \mathcal{X}_t \right) \leq \lambda_{T_1}\right) \\ =&\mathbb{P}\left(\frac{c_R+3}{2c_R} N\max_{j \in [N]} \left\|\mathcal{M}_j\left(\frac{1}{T} \sum_{t\in [T_1]} \epsilon_t \mathcal{X}_t\right)\right\| \leq \lambda_{T_1}\right) \\
 \geq & 1-\sum_{j \in [N]} \mathbb{P}\left(\left\|\mathcal{M}_j\left( \sum_{t\in [T_1]} \epsilon_t \mathcal{X}_t\right)\right\| \geq \frac{2c_R\lambda_{T_1}  T_1}{ (c_R + 3)N}\right)\\
 = & 1-\sum_{j \in [N]} \mathbb{P}\left(\left\|\sum_{t\in [T_1]} \mathcal{M}_j\left(  \epsilon_t \mathcal{X}_t\right)\right\| \geq \frac{2c_R\lambda_{T_1} T_1}{ (c_R + 3)N}\right)\\
 \geq & 1-\delta.
\end{align}
We claim that it suffices to prove the inequality $$
\mathbb{P}\left(\left\|\sum_{t\in [T_1]} \mathcal{M}_j\left( \epsilon_t \mathcal{X}_t \right)\right\| \geq \frac{2c_R \lambda_{T_1}  T_1}{(c_R + 3)N}\right) \leq \frac{\delta}{N}.
$$
Then we define the event $\mathcal{E}:=\left\{\max _{t\in [T_1]}\left|\epsilon_t\right|<v\right\}$, and by the definition of sub-Gaussian random variables, let $v=R \sqrt{2 \log (4T_1N / \delta)}$, we have
$
\mathbb{P}\left(\mathcal{E}^c\right)=\mathbb{P}\left(\max _{t\in [T_1]}\left|\epsilon_t\right|>v\right) \leq \sum_{t\in [T_1]} \mathbb{P}\left(\left|\epsilon_t\right|>v\right) \leq \delta / 2N
$. Under the event $\mathcal{E}$, we have the following bounds
\begin{align}
   &\left\|\mathcal{M}_j\left(  \epsilon_t \mathcal{X}_t\right) \right\| \leq\left\|\mathcal{M}_j\left(  \epsilon_t \mathcal{X}_t\right)\right\|_F=\left\|  \epsilon_t \mathcal{X}_t\right\|_F\\
   \leq & \max_{t}|\epsilon_t| \left\|  \mathcal{X}_t\right\|_F \leq v, \\  &\left\|\mathbb{E}\left(\mathcal{M}_j\left(  \epsilon_t \mathcal{X}_t\right) \mathcal{M}_j^\top \left(  \epsilon_t \mathcal{X}_t\right) \right)\right\| \\
\leq &\mathbb{E} \left\|\mathcal{M}_j\left(  \epsilon_t \mathcal{X}_t\right) \mathcal{M}_j^\top \left(  \epsilon_t \mathcal{X}_t\right) \right\|\\
    \leq &\mathbb{E} \left\|\mathcal{M}_j\left(  \epsilon_t \mathcal{X}_t\right) \right\|_F^2 \leq v^2.
\end{align}
Similarly, we can obtain $\left\|\mathbb{E}\left(\mathcal{M}_j^\top\left(  \epsilon_t \mathcal{X}_t\right) \mathcal{M}_j \left(  \epsilon_t \mathcal{X}_t\right) \right)\right\|\leq v^2$. Therefore, applying the matrix Bernstein inequality, we have
\begin{align}
&\mathbb{P}\left(\left\|\sum_{t\in [T_1]} \mathcal{M}_j\left(  \epsilon_t \mathcal{X}_t\right)\right\| \geq \frac{2c_R\lambda_{T_1}  T_1}{ (c_R + 3)N},\mathcal{E}\right) \\
\leq & \left(d+d^{N-1}\right) \exp \left(\frac{-(2c_R\lambda_{T_1}  T_1)^2 / 2(c_R+3)^2N^2}{T_1 v^2+2c_R\lambda_{T_1}  T_1v / 3(c_R+3)N}\right).
\end{align}
If we take $\lambda_{T_1}=\frac{R(c_R+3)N}{2c_R \sqrt{T_1}} \sqrt{ 2\log \frac{4T_1N}{\delta} \log\frac{2N(d+d^{N-1})}{\delta}}$, then 
\begin{align}
   &\mathbb{P}\left(\left\|\sum_{t\in [T_1]} \mathcal{M}_j\left( \epsilon_t \mathcal{X}_t \right)\right\| \geq \frac{2c_R \lambda_{T_1}  T_1}{(c_R+3)N}\right)\\
   \leq & \mathbb{P}\left(\left\|\sum_{t\in [T_1]} \mathcal{M}_j\left( \epsilon_t \mathcal{X}_t \right)\right\| \geq \frac{2c_R \lambda_{T_1}  T_1}{(c_R+3)N},\mathcal{E}\right) +\mathbb{P}\left(\mathcal{E}^c\right)\\
   \leq & \frac{\delta}{2N}+ \frac{\delta}{2N} \leq \frac{\delta}{N}. 
\end{align}

\end{proof}

\section{Proof of Corollary \ref{coro1}}
\begin{proof}
   \begin{align}
    \mathbb{E}(R_T)\leq & 2k_\mu T_1+12k_\mu T \frac{(1+c_R)\sqrt{ \phi} \lambda_{T_1}}{(3+c_R)c_l k_u}\\
    = & 2k_\mu T_1+12k_\mu T \frac{(1+c_R)\sqrt{ \phi} }{(3+c_R)c_l k_u} \frac{R(c_R+3)N}{2c_R \sqrt{T_1}} \notag \\ 
    & \sqrt{ 2\log \frac{4T_1N}{\delta} \log\frac{2N(d+d^{N-1})}{\delta}}\\
    = & 2k_\mu T_1+12 T \frac{(1+c_R) RN\sqrt{ \phi} }{ 2c_R c_l \sqrt{T_1} } \notag \\ 
    & \sqrt{ 2\log \frac{4T_1N}{\delta} \log\frac{2N(d+d^{N-1})}{\delta}}\\
    = & 2k_\mu T_1+12 T \frac{(1+c_R) RN\sqrt{ 2r} }{ 2c_R c_l \sqrt{T_1} } \notag \\ 
    & \sqrt{ 2\log \frac{4T_1N}{\delta} \log\frac{2N(d+d^{N-1})}{\delta}}\\
    =& \tilde{O}\left(c_l^{-\frac{2}{3}} r^{\frac{1}{3}} T^{\frac{2}{3}}\right)\\
    =& \tilde{O}\left(d^\frac{N}{3} r^{\frac{1}{3}} T^{\frac{2}{3}}\right).
\end{align} 
\end{proof}

\section{Proof of Lemma \ref{lem2}}

When considering slices that exhibit both sparsity and low-rankness, we again take the $(1, 2)$-slice as an example. In this case, the true parameters consist of only $s$ nonzero slices \(\Theta^*_{\cdot \cdot k}, k \in [d_3]\), all of which are low-rank matrices. The corresponding parameter space is given by \(\{\Theta \in \mathbb{R}^{d_1 \times d_2 \times d_3} \mid \sum_{k \in [d_3]}\operatorname{rank}(\Theta_{\cdot \cdot k}) \leq r\}\). Therefore, we choose \(R(\Theta) = \|\Theta\|_{(1,2),*} = \sum_{k \in [d_3]} \|\Theta_{\cdot \cdot k}\|_{*}\). In this case, \(R^*(\Theta) = \max_{k \in [d_3]} \|\Theta_{\cdot \cdot k}\|\).

\begin{proof}
At this point, the adjustment parameters need to satisfy the following inequality:
\begin{align}
    &\mathbb{P}\left(\frac{c_R+3}{2c_R}R^*\left(\frac{1}{T_1} \sum_{t\in[T_1]} \epsilon_t \mathcal{X}_t \right) \leq \lambda_{T_1}\right) \\ = &\mathbb{P}\left(\frac{c_R+3}{2c_R} \max_{k \in [d_3]} \left\|\left[\frac{1}{T_1} \sum_{t\in[T_1]} \epsilon_t \mathcal{X}_t\right]_{\cdot \cdot k}\right\| \leq \lambda_{T_1}\right) \\
 \geq & 1 - \sum_{k \in [d_3]} \mathbb{P}\left(\left\|\left[ \sum_{t\in[T_1]} \epsilon_t \mathcal{X}_t\right]_{\cdot \cdot k}\right\| \geq \frac{2c_R\lambda_{T_1} T_1}{c_R+3}\right)\\
 \geq & 1-\delta.
\end{align}
We claim that it suffices to prove the inequality $
\mathbb{P}\left(\left\|\left[ \sum_{t\in[T_1]} \epsilon_t \mathcal{X}_t\right]_{\cdot \cdot k}\right\| \geq \frac{2c_R\lambda_{T_1} T_1}{c_R+3}\right) \leq \frac{\delta}{d_3}
$. Here we define the event $\mathcal{E}:=\left\{\max _{t\in[T_1]}\left|\epsilon_t\right|<v\right\}$. We take $v=R \sqrt{2 \log (4T_1d_3 / \delta)}$, then $\mathbb{P}\left(\mathcal{E}^c\right) \leq \delta / 2 d_3$. Under the event $\mathcal{E}$, we can get 
$$
\left\|\left[  \epsilon_t \mathcal{X}_t\right]_{\cdot \cdot k} \right\| \leq\left\|\left[  \epsilon_t \mathcal{X}_t\right]_{\cdot \cdot k}\right\|_F\leq \max_{t}|\epsilon_t| \left\|  \left[  \mathcal{X}_t\right]_{\cdot \cdot k}\right\|_F \leq v,
$$
\begin{align}
    \left\|\mathbb{E}\left(\left[  \epsilon_t \mathcal{X}_t\right]_{\cdot \cdot k} \left[  \epsilon_t \mathcal{X}_t\right]_{\cdot \cdot k}^\top \right)\right\| \leq \mathbb{E} \left\|\left[  \epsilon_t \mathcal{X}_t\right]_{\cdot \cdot k} \left[  \epsilon_t \mathcal{X}_t\right]_{\cdot \cdot k}^\top\right\| \\
    \leq \mathbb{E} \left\|\left[  \epsilon_t \mathcal{X}_t\right]_{\cdot \cdot k} \right\|_F^2 \leq v^2.
\end{align}

Similarly, we can obtain $\left\|\mathbb{E}\left( \left[  \epsilon_t \mathcal{X}_t\right]_{\cdot \cdot k}^\top \left[  \epsilon_t \mathcal{X}_t\right]_{\cdot \cdot k}\right)\right\|\leq v^2$. Then according to the matrix Bernstein inequality, we have
\begin{align}
&\mathbb{P}\left(\left\|\sum_{t\in[T_1]} \left[  \epsilon_t \mathcal{X}_t\right]_{\cdot \cdot k} \right\| \geq \frac{2c_R\lambda_{T_1}  T_1}{ c_R+3},\mathcal{E}\right) \\
\leq &\left(d_1+d_2\right) \exp \left(\frac{-(2c_R\lambda_{T_1}  T_1)^2 / 2(c_R+3)^2}{T_1 v^2+2c_R\lambda_{T_1}  T_1v / 3(c_R+3)}\right).
\end{align}
If we take $\lambda_{T_1}=\frac{R(c_R+3)}{c_R\sqrt{T_1}} \sqrt{ \log \frac{4T_1d_3}{\delta} \log\frac{2d_3(d_1+d_2)}{\delta}}$, then 
\begin{align}
 &\mathbb{P}\left(\left\|\left[ \sum_{t\in[T_1]} \epsilon_t \mathcal{X}_t\right]_{\cdot \cdot k}\right\| 
 \geq  \frac{2c_R\lambda_{T_1} T_1}{c_R+3}\right)\\
 \leq &\mathbb{P}\left(\left\|\sum_{t\in[T_1]} \left[  \epsilon_t \mathcal{X}_t\right]_{\cdot \cdot k} \right\| \geq \frac{2c_R\lambda_{T_1}  T_1}{ c_R+3},\mathcal{E}\right) +\mathbb{P}\left(\mathcal{E}^c\right)\\
 \leq & \frac{\delta}{d_3}.   
\end{align}

\end{proof}

\section{Proof of Corollary \ref{coro2}}
\begin{proof}
   \begin{align}
    \mathbb{E}(R_T)\leq & 2k_\mu T_1+12k_\mu T \frac{(1+c_R)\sqrt{ \phi} \lambda_{T_1}}{(3+c_R)c_l k_u}\\
    = & 2k_\mu T_1+12k_\mu T \frac{(1+c_R)\sqrt{ \phi} }{(3+c_R)c_l k_u} \frac{R(c_R+3)}{c_R \sqrt{T_1}}\notag \\ 
    & \sqrt{ \log \frac{4T_1d_3}{\delta} \log\frac{2d_3(d_1+d_2)}{\delta}}\\
    = & 2k_\mu T_1+12 T \frac{(1+c_R) R\sqrt{ \phi} }{ c_R c_l \sqrt{T_1} } \notag \\ 
    & \sqrt{ \log \frac{4T_1d_3}{\delta} \log\frac{2d_3(d_1+d_2)}{\delta}}\\
    = & 2k_\mu T_1+12 T \frac{(1+c_R) R\sqrt{ 2r} }{ c_R c_l \sqrt{T_1} } \notag \\ 
    & \sqrt{ \log \frac{4T_1d_3}{\delta} \log\frac{2d_3(d_1+d_2)}{\delta}}\\
    =& \tilde{O}\left(c_l^{-\frac{2}{3}} r^{\frac{1}{3}} T^{\frac{2}{3}}\right)\\
    =& \tilde{O}\left(d r^{\frac{1}{3}} T^{\frac{2}{3}}\right).
\end{align} 
\end{proof}

\section{Extension to Entry-Wise Sparsity}

\subsection{Results}
We consider the problem of entry-wise sparse tensor bandits, assuming that the unknown tensor parameter $\Theta^*$ is entry-wise sparse, meaning it has only $s$ non-zero elements, with $s$ much smaller than $d_1 d_2 d_3$. This naturally leads to using $R(\Theta)=\|\Theta\|_1=\sum_{i\in [d_1],} \sum_{j \in [d_2]} 
 \sum_{k \in [d_3]} 
 |\Theta_{ijk}|$ regularization. In this case, we can clarify the corresponding quantities, i.e., $c_R=1$, $\phi=s$.

The following lemma provides theoretical guidance for $\lambda_{T_1}$ in this specific tensor bandits problem.

\begin{lemma}\label{lem8}
    For any $\delta \in(0,1)$, let $\alpha=\frac{c_R+3}{2c_R}$, use
$$
\lambda_{T_1}=\frac{c\alpha R k}{\sqrt{T_1}} \sqrt{ \log (2 d_1d_2d_3 / \delta)}
$$
in Algorithm \ref{algo1} with $R(\Theta)=\|\Theta\|_1$, then with probability at least $1-\delta$, we have $\lambda_{T_1}  \geq \alpha R^*\left(\frac{1}{T_1} \sum_{t \in [T_1]} \epsilon_t \mathcal{X}_t\right)$.
\end{lemma}

Given the above lemma, it is easy to apply Theorem \ref{theo2} to obtain a specific regret bound in the entry-wise sparsity tensor bandits. 

\begin{corollary}\label{coro3}
Under Assumption \ref{as1}-\ref{as3}, let $
\lambda_{T_1}=\frac{c\alpha R k}{\sqrt{T_1}} \sqrt{ \log (2 d_1d_2d_3 / \delta)}
$, then the expected cumulative regret of Algorithm \ref{algo1} with the entry-wise sparsity is upper bounded by
$$
\mathbb{E}(R_T)=O\left(s^{1 / 3} T^{2 / 3} \log^{1/3} (d_1d_2d_3 )\right).
$$

\end{corollary}

In this case, it can degenerate into the Lasso bandits problem. Therefore, we compare it with the related Lasso bandits results obtained after vectorization. Specifically, whether compared with the bound established by \cite{hao2020high} of $\mathcal{O}(C_{\min}^{-2/3}s^{2/3} T^{2/3})$ or the bound established by \cite{jang2022popart} of $\mathcal{O}(H_{*}^{2/3}s^{2/3} T^{2/3})$, considering that $C_{\min}$ and $H_{*}$ are related to the dimension in certain cases, both of these bounds exhibit polynomial dependence on the dimension. In contrast, the bound in this paper achieves a logarithmic order on the dimension, which provides a significant advantage in high-dimensional settings. Furthermore, it is important to emphasize that the problem setting in this paper is broader than the aforementioned studies, as it encompasses generalized linear models, including their linear case.

\subsection{Proof}
\subsubsection{Proof of Lemma \ref{lem8}}
In the entry-wise sparse tensor bandits problem, the true parameter $\Theta^*$ has only $s$ nonzero elements, with $s$ being much smaller than $d_1 d_2 d_3$. Therefore, using $R(\Theta) = \|\Theta\|_1 = \sum_{i\in [d_1]} \sum_{j \in [d_2]} \sum_{k \in [d_3]} |\Theta_{ijk}|$ as the regularization term is a natural choice. Furthermore, we have $R^*(\Theta) = \|\Theta\|_\infty = \max_{i \in [d_1], j \in [d_2], k \in [d_3]} |\Theta_{ijk}|$. 

\begin{proof}
In this case, the adjustment of parameters needs to satisfy the following inequality:
\begin{align}
&\mathbb{P}\left(\frac{c_R+3}{2c_R}R^*\left(\frac{1}{T_1} \sum_{t\in[T_1]} \epsilon_t \mathcal{X}_t \right) \leq \lambda_{T_1}\right) \\ 
 =& \mathbb{P}\left(\frac{c_R+3}{2c_R} \max_{ijk} \left|\left[\frac{1}{T_1} \sum_{t\in [T_1]}\epsilon_t \mathcal{X}_t\right]_{ijk}\right| \leq \lambda_{T_1}\right) \\
 \geq & 1-\sum_{ijk} \mathbb{P}\left(\left|\left[ \sum_{t\in [T_1]} \epsilon_t \mathcal{X}_t\right]_{ijk}\right| \geq \frac{2c_R\lambda_{T_1} T_1}{ c_R+3}\right)\\
\geq & 1-\delta.
\end{align}
The following proof shows that $\epsilon_t \left[\mathcal{X}_t\right]_{ijk}$ is an independently $k_0$-sub-exponential variable with respect to $t$. Note that $\left\{\operatorname{vec}\left(\mathcal{X}_t \right),t\in [T_1]\right\}$ are independent $k$-sub-Gaussian vectors, then we have $\left\{\left[\mathcal{X}_t\right]_{ijk},t\in [T_1]\right\}$ are independent $k$-sub-Gaussian variables. Also note that each $\epsilon_t$ is $R$-sub-Gaussian, we have $\left\{\epsilon_t \left[\mathcal{X}_t\right]_{ijk},t\in [T_1]\right\}$ are independent $Rk$-sub-exponential random variables. Therefore we can use Bernstein's inequality and get
\begin{align}
&\mathbb{P}\left(\left|\left[ \sum_{t\in [T_1]} \epsilon_t X_t\right]_{ijk}\right| \geq \frac{2c_RT_1\lambda_{T_1}}{ c_R+3}\right) \\
\leq & 2 \exp\left(-c \min \left\{\frac{4c_R^2T_1^2\lambda_{T_1}^2}{ (c_R+3)^2 T_1 k_0^2},\frac{2c_RT_1\lambda_{T_1}}{ (c_R+3)k_0}\right\}\right) ,
\end{align}
where $k_0=Rk$.
If we take $\lambda_{T_1}=\frac{(c_R+3)Rk}{2c c_R \sqrt{T_1}}\sqrt{\log \frac{2d_1d_2d_3}{\delta}}$, then we conclude that
\begin{align}
    &\mathbb{P}\left(\frac{c_R+3}{2c_R}R^*\left(\frac{1}{T} \sum_{t \in [T_1]} \epsilon_t \mathcal{X}_t \right) 
    \leq  \lambda_{T_1}\right)\\
    \geq & 1-\sum_{ijk} \frac{\delta}{d_1d_2d_3} \geq 1-\delta.
\end{align}

\end{proof}

\subsubsection{Proof of Corollary \ref{coro3}}
\begin{proof}
   \begin{align}
    \mathbb{E}(R_T)\leq & 2k_\mu T_1+12k_\mu T \frac{(1+c_R)\sqrt{ \phi} \lambda_{T_1}}{(3+c_R)c_l k_u}\\
    = & 2k_\mu T_1+12k_\mu T \frac{(1+c_R)\sqrt{ \phi} }{(3+c_R)c_l k_u} \frac{Rk(c_R+3)}{2c c_R \sqrt{T_1}} \notag\\ 
    & \sqrt{\log \frac{2d_1d_2d_3}{\delta}}\\
    = & 2k_\mu T_1+12 T \frac{(1+c_R) Rk\sqrt{ \phi} }{ 2c c_R c_l \sqrt{T_1} } \sqrt{\log \frac{2d_1d_2d_3}{\delta}}\\
    = & 2k_\mu T_1+6 T \frac{(1+c_R) Rk\sqrt{ s} }{c c_R c_l \sqrt{T_1} }  \sqrt{\log \frac{2d_1d_2d_3}{\delta}}\\
    =& O\left(k^{\frac{2}{3}} c_l^{-\frac{2}{3}} s^{\frac{1}{3}} T^{\frac{2}{3}} \log^{\frac{1}{3}} \frac{2d_1d_2d_3}{\delta}\right)\\
    =& O\left(s^{\frac{1}{3}} T^{\frac{2}{3}} \log^{\frac{1}{3}} \frac{2d_1d_2d_3}{\delta}\right)\\
    =& \tilde{O}\left( s^{\frac{1}{3}} T^{\frac{2}{3}}\right).
\end{align} 
\end{proof}

\section{Extension to Fiber-Wise Sparsity}
\subsection{Results}
We focus on the fiber-wise sparse tensor bandits problem. In this case, taking mode-1 fibers as an example, i.e., $\Theta^* \in S(\Theta^*)=\{ \cdot \in [d_1] \ | \ \Theta^*_{\cdot j k} \neq 0 \}$, and $|S(\Theta^*)| = s \ll d_1$. Then we use the regularization norm $R(\Theta)=\|\Theta\|_{1,q}= \sum_{j \in [d_2]} 
 \sum_{k \in [d_3]} \left[\sum_{i \in [d_1]}|\Theta_{i jk}|^q\right]^{\frac{1}{q}}$, $q>1$. Then, at this point, $c_R = 1$ and $ \phi = \eta^2\left(d_1, \frac{1}{q}-\frac{1}{2}\right) s $, where $\eta\left(\cdot, m\right)=\max \left\{1, \cdot^m\right\}$. Likewise, we present the following lemma to identify a suitable choice for the parameter $ \lambda_{T_1}$ in this tensor bandit.

\begin{lemma}\label{lem9}
    For any $\delta \in(0,1)$, let $\alpha=\frac{c_R+3}{2c_R}$, use
    \begin{align*}
        \lambda_{T_1}=&\frac{c\alpha Rk(\sqrt{d_1}+\sqrt{\log \frac{4d_2d_3}{\delta}})}{\sqrt{T_1}} \\
        &\sqrt{2 \log (\frac{4T_1d_2 d_3}{ \delta})}\eta\left(d_1, \frac{1}{2}-\frac{1}{q}\right) 
    \end{align*}
in Algorithm \ref{algo1} with $R(\Theta)=\|\Theta\|_{1,q}$, then with probability at least $1-\delta$, we have $\lambda_{T_1} \geq \alpha R^*\left(\frac{1}{T_1} \sum_{t \in [T_1]} \epsilon_t \mathcal{X}_t\right)$.
\end{lemma}

Based on the above lemma, we can similarly derive the regret bounds for this low-dimensional structure.

\begin{corollary}\label{coro4}
Under Assumption \ref{as1}-\ref{as3}, let \begin{align*}
        \lambda_{T_1}=&\frac{c\alpha Rk(\sqrt{d_1}+\sqrt{\log \frac{4d_2d_3}{\delta}})}{\sqrt{T_1}} \\
        &\sqrt{2 \log (\frac{4T_1d_2 d_3}{ \delta})}\eta\left(d_1, \frac{1}{2}-\frac{1}{q}\right) ,
    \end{align*} then the expected cumulative regret of the Algorithm \ref{algo1} in the tensor bandits problem with fiber-wise sparsity is upper bounded by
\begin{align*}
\mathbb{E}(R_T)= &\tilde{O}\left(  \eta^{\frac{4}{3}}\left(d_1, \frac{1}{q}-\frac{1}{2}\right) \max \left\{d_1,\log \frac{4d_2d_3}{\delta}\right\} ^{\frac{1}{3}} \right. \notag \\
    & \left. s^{\frac{1}{3}} T^{\frac{2}{3}}\right).
\end{align*}

\end{corollary}

The above scenario can degenerate into the group-sparse matrix bandits problem, which characterizes the low-dimensional structure of group sparsity. In particular, when $q=2$, $R(\Theta)$ corresponds to group Lasso regularization. In this case, compared to the latest group Lasso bandits results applied after matricization, the regret bound of \cite{li2022simple} is $\mathcal{O}\left(\max\{d_1, \log(d_2d_3)\}^{\frac{1}{2}} s^{\frac{1}{3}} T^{\frac{2}{3}} \right)$, while the bound in this paper is $\tilde{O}\left( \max \left\{d_1,\log (d_2d_3)\right\}^{\frac{1}{3}}  s^{\frac{1}{3}} T^{\frac{2}{3}}\right)$. It is important to emphasize that the theoretical results in this paper have a smaller dependence on the dimension, providing a greater advantage in high-dimensional settings.

\subsection{Proof}
\subsubsection{Proof of Lemma \ref{lem9}}
When considering a specific low-dimensional structure, specifically fiber-wise sparsity, we take the mode-1 fibers as an example. In this case, we have $\Theta^* \in S(\Theta^*) = \{ \cdot \in [d_1] \ | \ \Theta^*_{\cdot j k} \neq \boldsymbol{0} \}$, and $|S(\Theta^*)| = s \ll d_1$. We then use the regularization norm $R(\Theta) = \|\Theta\|_{1,q} = \sum_{j \in [d_2]} \sum_{k \in [d_3]} \left[\sum_{i \in [d_1]} |\Theta_{i jk}|^q\right]^{\frac{1}{q}}$, where $q > 1$. Correspondingly, we have $R^*(\Theta) = \|\Theta\|_{\infty,p} = \max_{j \in [d_2],k \in [d_3]}  \left[\sum_{i \in [d_1]} |\Theta_{i jk}|^p\right]^{\frac{1}{p}}$, where $\frac{1}{q} + \frac{1}{p} = 1$. 

\begin{proof}
    In this setup, the requirement for the adjusting parameters is that the following inequality must be satisfied:
    \begin{align}
        &\mathbb{P}\left(\frac{c_R+3}{2c_R}R^*\left(\frac{1}{T_1} \sum_{t\in [T_1]} \epsilon_t \mathcal{X}_t \right) \leq \lambda_{T_1}\right) \\ =&\mathbb{P}\left(\frac{c_R+3}{2c_R} \max_{jk} \left\|\left[\frac{1}{T_1} \sum_{t\in [T_1]}\epsilon_t \mathcal{X}_t\right]_{\cdot jk}\right\|_{p} \leq \lambda_{T_1}\right) \\
\geq& 1-\sum_{jk} \mathbb{P}\left(\left\|\left[ \sum_{t\in [T_1]} \epsilon_t \mathcal{X}_t\right]_{\cdot jk}\right\|_{p} \geq \frac{2c_R\lambda_{T_1} T_1}{ c_R+3}\right)\\
\geq & 1-\delta.
\end{align}
We claim that it suffices to prove the inequality $$\mathbb{P}\left(\left\|\left[ \sum_{t\in [T_1]} \epsilon_t \mathcal{X}_t\right]_{\cdot jk}\right\|_{p} \geq \frac{2c_R\lambda_{T_1} T_1}{ c_R+3}\right) \leq \frac{\delta}{d_2 d_3}.$$
Here we define the event $\mathcal{E}:=\left\{\max _{t\in [T_1]}\left|\epsilon_t\right|<v\right\}$, and by the definition of sub-Gaussian random variables, if we take $v=R \sqrt{2 \log (4T_1 d_2d_3/ \delta)}$, then $\mathbb{P}\left(\mathcal{E}^c\right)=\mathbb{P}\left(\max _{t\in [T_1]}\left|\epsilon_t\right|>v\right) \leq \sum_{t\in [T_1]} \mathbb{P}\left(\left|\epsilon_t\right|>v\right) \leq \delta / 2 d_2d_3$. Under the event $\mathcal{E}$, we can get
\begin{align}
    &\mathbb{P}\left(\left\|\left[ \sum_{t\in [T_1]} \epsilon_t \mathcal{X}_t\right]_{\cdot j k}\right\|_{p} \geq \frac{2c_R\lambda_{T_1}  T_1}{ c_R+3}, \mathcal{E} \right) \\
    \leq & \mathbb{P}\left(\left\|\left[ \sum_{t\in [T_1]} \mathcal{X}_t\right]_{\cdot j k}\right\|_{p} \geq \frac{2c_R\lambda_{T_1}  T_1}{ v(c_R+3)} \right).
\end{align}
Let $\left\|\sum_{t\in [T_1]} \left[\mathcal{X}_t\right]_{\cdot j k}\right\|_{p}:=\sup_{\|a\|_{q}  \leq 1} \left\langle \sum_{t\in [T_1]} \left[\mathcal{X}_t\right]_{\cdot j k},a  \right\rangle := \sup_{\|a\|_{q} \leq 1} X_a $, then $X_a$ is a random process defined on the subset $\mathbb{A}=\{a\in \mathbb{R}^{d_1} |\|a\|_{q} \leq 1\}$, and
\begin{align}
    &\mathbb{P}\left(\left\|\left[ \sum_{t\in [T_1]} \mathcal{X}_t\right]_{\cdot j k}\right\|_{p} \geq \frac{2c_R\lambda_{T_1}  T_1}{ v(c_R+3)} \right)\\
    =&\mathbb{P}\left(\sup_{\|a\|_{q} \leq 1} X_a \geq \frac{2c_R\lambda_{T_1}  T_1}{ v(c_R+3)} \right)\\
    =&\mathbb{P}\left(\sup_{\|a\|_{q} \leq 1} |X_a| \geq \frac{2c_R\lambda_{T_1}  T_1}{ v(c_R+3)} \right).
\end{align}
The following proves that $\sum_{t\in [T_1]} \left[\mathcal{X}_t\right]_{\cdot j k}$ is a sub-Gaussian vector. Note that $\operatorname{vec}\left(\mathcal{X}_t\right)$ is the sub-Gaussian vector, then $\left[\mathcal{X}_t\right]_{ijk}$ is the sub-Gaussian variable with sub-Gaussian norm $k$. According to Lemma \ref{ref.lem5}, we have $\sum_{t\in [T_1]}\left[\mathcal{X}_t\right]_{ijk}$ is the sub-Gaussian variable with sub-Gaussian norm $k\sqrt{T_1}$. For any fixed $z \in \mathbb{S}^{d_1-1 }$, according to the definition of the sub-Gaussian vector, we have $\langle \sum_{t\in [T_1]} \left[\mathcal{X}_t\right]_{\cdot j k}, z \rangle=\sum_{i \in [d_1]} \sum_{t\in [T_1]} \left[\mathcal{X}_t\right]_{ij k} z_{i}$ is the sub-Gaussian variable with sub-Gaussian norm $ck\sqrt{T_1}$. Thus,
\begin{align}
    &\|X_{a_1}-X_{a_2}\|_{\psi_2} \\=&\|\langle \sum_{t\in [T_1]} \left[\mathcal{X}_t\right]_{\cdot j k},a_1  \rangle-\langle \sum_{t\in [T_1]} \left[\mathcal{X}_t\right]_{\cdot j k},a_2  \rangle\|_{\psi_2} \\
    =&\|\langle \sum_{t\in [T_1]} \left[\mathcal{X}_t\right]_{\cdot j k},a_1-a_2 \rangle\|_{\psi_2}\\
    \leq & \| \sum_{t\in [T_1]} \left[\mathcal{X}_t\right]_{\cdot j k}\|_{\psi_2} \|a_1-a_2\|_2= c \sqrt{T}k \|a_1-a_2\|_2.
\end{align}
Next, we will calculate $\operatorname{rad}(\mathbb{A})$ and $w(\mathbb{A})$ separately. Let $g \sim \mathcal{N}(0,I)$, according to Lemma \ref{ref.lem6}, we have 
\begin{align}
    &w(\mathbb{A})=\mathbb{E}\left( \sup_{a \in \mathbb{A}} \langle a,g\rangle\right)=\mathbb{E} \|g\|_p\\
    \leq & \max\{1,d_1^{1/p-1/2}\} \mathbb{E}  \|g\|_2 \\
    \leq & \max\{1,d_1^{1/p-1/2}\} \sqrt{\mathbb{E} \|g\|_2^2}\\
    =&\max\{1,d_1^{1/p-1/2}\}\sqrt{d_1 }\\
    =&\max\{d_1^{1/2},d_1^{1/p}\}=\max\{d_1^{1/2},d_1^{1-1/q}\},\\
   &\operatorname{rad}(\mathbb{A})=\sup_{a \in \mathbb{A}} \| a\|_2\leq  \max\{1,d_1^{1/2-1/q}\} \sup_{a \in \mathbb{A}} \| a\|_q\\
   \leq &\max\{1,d_1^{1/2-1/q}\}. 
\end{align}
Let $u=\sqrt{\log \frac{4d_2d_3}{\delta}}$, according to Lemma \ref{ref.lem3}, we have
\begin{align}
    \mathbb{P} \left(\sup _{a \in \mathbb{A}}\left|X_a\right| \geq C k \sqrt{T}(\max\{d_1^{1/2},d_1^{1-1/q}\} \right. \notag \\ \left. +u \max\{1,d_1^{1/2-1/q}\}) \right) \leq 2 \exp \left(-u^2\right).
\end{align}
Let $\alpha=\frac{c_R+3}{2c_R}$, $\lambda_{T_1}=\frac{c\alpha Rk(\sqrt{d_1}+\sqrt{\log \frac{4d_2d_3}{\delta}})}{\sqrt{T_1}} \sqrt{2 \log (\frac{4T_1d_2 d_3}{ \delta})}\max\{1,d_1^{1/2-1/q}\}$, then
\begin{align}
    &\mathbb{P}\left(\left\|\left[ \sum_{t\in [T_1]} \epsilon_t \mathcal{X}_t\right]_{\cdot j k}\right\|_p \geq \frac{2c_R\lambda_{T_1}  T_1}{ c_R+3} \right)\\
    \leq & \mathbb{P}\left(\left\|\left[ \sum_{t\in [T_1]} \epsilon_t \mathcal{X}_t\right]_{\cdot j k}\right\|_p \geq \frac{2c_R\lambda_{T_1}  T_1}{ c_R+3}, \mathcal{E} \right)  +\mathbb{P}\left(\mathcal{E}^c\right)\\
    \leq &\frac{\delta}{d_2d_3}.
\end{align}

\end{proof}

\subsubsection{Proof of Corollary \ref{coro4}}
\begin{proof}
   \begin{align}
    \mathbb{E}(R_T)\leq & 2k_\mu T_1+12k_\mu T \frac{(1+c_R)\sqrt{ \phi} \lambda_{T_1}}{(3+c_R)c_l k_u}\\
    = & 2k_\mu T_1+12k_\mu T \frac{(1+c_R)\sqrt{ \phi} }{(3+c_R)c_l k_u} \max\{1,d_1^{1/2-1/q}\} \notag \\ 
    &\frac{c\alpha Rk(\sqrt{d_1}+\sqrt{\log \frac{4d_2d_3}{\delta}})}{\sqrt{T_1}}  \sqrt{2 \log (\frac{4T_1d_2 d_3}{ \delta})}\\
    = & 2k_\mu T_1+12 T \frac{(1+c_R)\sqrt{ \phi} }{2c_R c_l } \max\{1,d_1^{1/2-1/q}\}\notag \\ 
    &\frac{c  Rk(\sqrt{d_1}+\sqrt{\log \frac{4d_2d_3}{\delta}})}{\sqrt{T_1}}  \sqrt{2 \log (\frac{4T_1d_2 d_3}{ \delta})}\\
    = & 2k_\mu T_1+6 T \frac{(1+c_R)c Rk\sqrt{ s} }{ c_R c_l \sqrt{T_1} } \eta^2 \left(d_1, \frac{1}{q}-\frac{1}{2}\right)\notag \\ 
    &\left(\sqrt{d_1}+\sqrt{\log \frac{4d_2d_3}{\delta}}\right)  \sqrt{2 \log (\frac{4 T_1 d_2 d_3}{ \delta})}\\
    =& \tilde{O}\left(\eta^{\frac{4}{3}}\left(d_1, \frac{1}{q}-\frac{1}{2}\right) \max \left\{d_1,\log \frac{4d_2d_3}{\delta}\right\}^{\frac{1}{3}} \right. \notag \\
    &\left. k^{\frac{2}{3}} c_l^{-\frac{2}{3}}  s^{\frac{1}{3}} T^{\frac{2}{3}}\right)\\
    =& \tilde{O}\left( \eta^{\frac{4}{3}}\left(d_1, \frac{1}{q}-\frac{1}{2}\right) \max \left\{d_1,\log \frac{4d_2d_3}{\delta}\right\} ^{\frac{1}{3}} \right. \notag \\
    & \left. s^{\frac{1}{3}} T^{\frac{2}{3}}\right).
\end{align} 
\end{proof}

\section{Additional Experiments}

\subsection{Tensor-Wise Low-Rankness}
\begin{figure}[h]
    \centering
    \subfigure[Cumulative Regret $R_T$]{
        \label{Fig.sub.5}
        \includegraphics[width=0.20\textwidth]{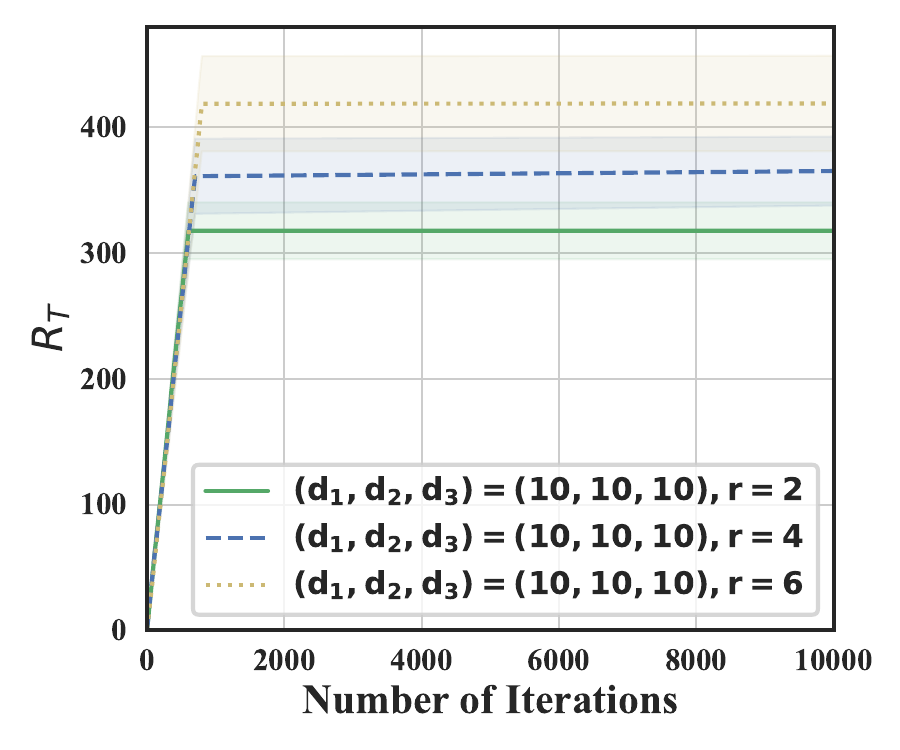}}
        \hfil
    \subfigure[The ratio $R_T/B_T$]{
        \label{Fig.sub.6}
        \includegraphics[width=0.20\textwidth]{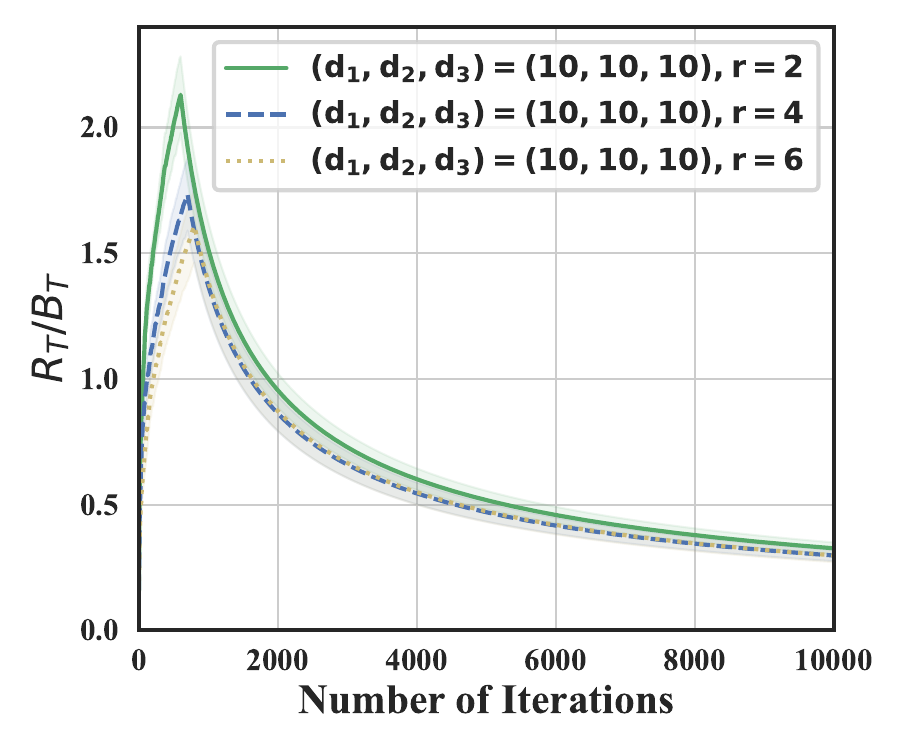}}
    \caption{Experimental results of tensor bandits under low multi-linear rankness with different rank settings. (a) displays the curve of cumulative regret over time, while (b) shows the variation of the ratio of cumulative regret to the theoretical bound $B_T$ over time.}
    \label{fig:figure3}
\end{figure}
In the case of tensors with low multi-linear rankness, we first generate the true parameters $\Theta^* \in \mathbb{R}^{d_1 \times d_2 \times d_3}$, with its elements drawn from a standard normal distribution. We then project it into a low multi-linear rank space. Next, we continue to validate the effectiveness of the algorithm from the two previously mentioned perspectives. First, we check whether the algorithm can achieve a sublinear cumulative regret upper bound. In three different different ranks (as shown in Figure \ref{Fig.sub.5}), the upper bound of cumulative regret exhibits stability as the number of decision rounds increases, demonstrating ideal sublinear growth. Additionally, the curves in the figures indicate that cumulative regret also rises with increasing rank, which is consistent with the theoretical results. Second, we verify whether the algorithm can achieve the theoretical cumulative regret upper bound, specifically by examining the ratio of cumulative regret to the theoretical bound. As shown in Figure \ref{Fig.sub.6}, this ratio gradually stabilizes at a constant less than 1 as the number of rounds increases, indicating that our algorithm can achieve the theoretical cumulative regret upper bound.
\subsection{Slice-Wise Sparsity and Low-Rankness}
\begin{figure}[h]
    \centering
    \subfigure[Cumulative Regret $R_T$]{
        \label{Fig.sub.7}
        \includegraphics[width=0.20\textwidth]{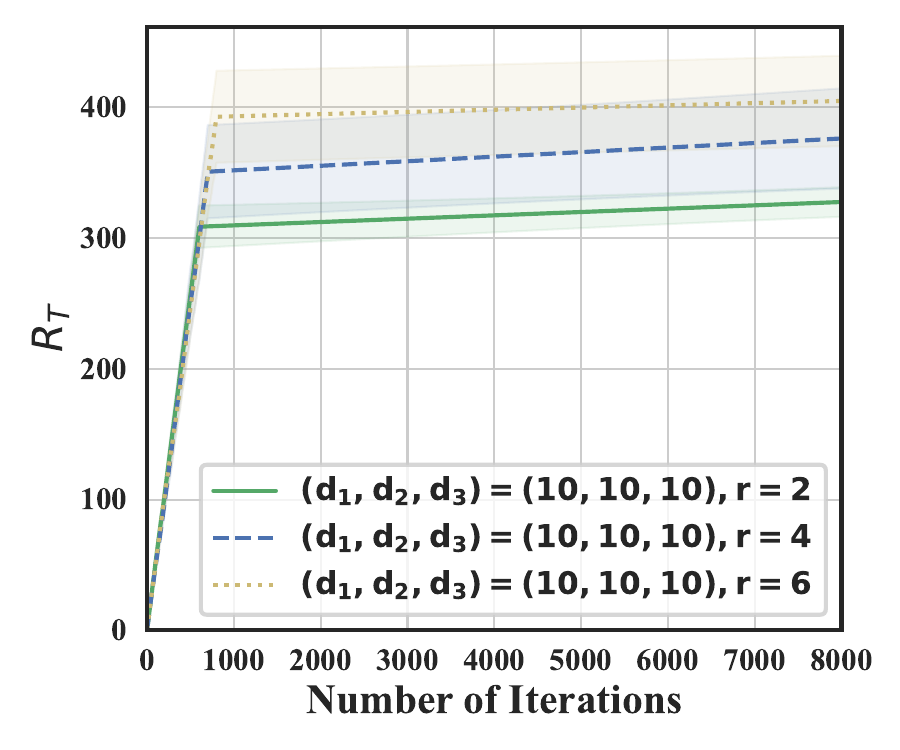}} 
        \hfil
    \subfigure[The ratio $R_T/B_T$]{
        \label{Fig.sub.8}
        \includegraphics[width=0.20\textwidth]{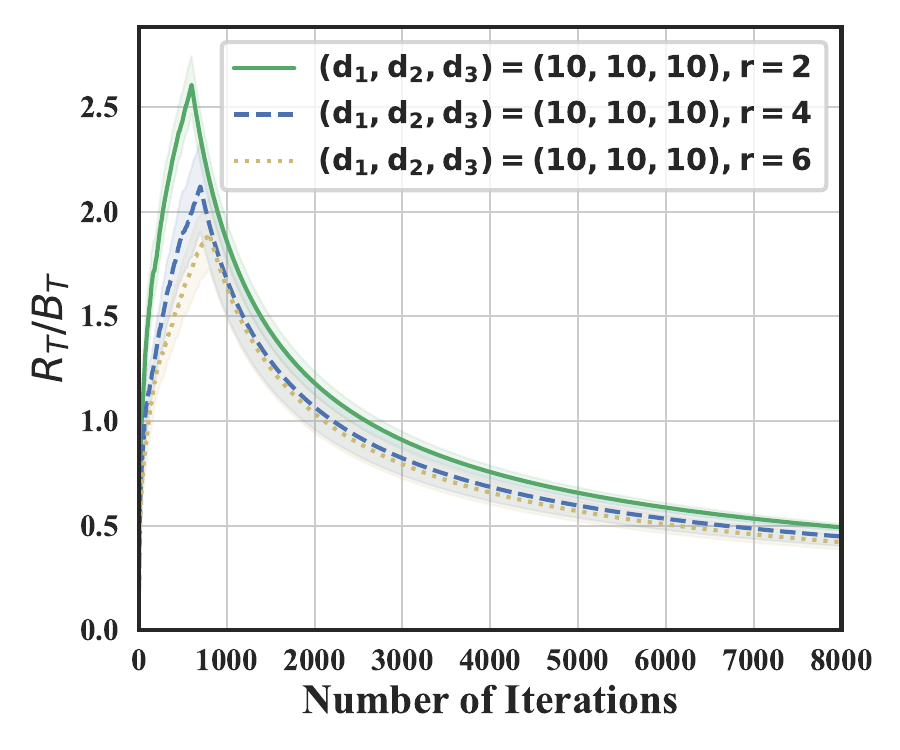}}
    \caption{Experimental results of tensor bandits under slice sparse and low rank structure with different rank settings. (a) displays the curve of cumulative regret over time, while (b) shows the variation of the ratio of cumulative regret to the theoretical bound $B_T$ over time.}
    \label{fig:figure4}
\end{figure}
For low-dimensional structures considering slice sparsity and low-rankness, the relevant experimental results at different dimensions have been presented in the experimental section of the main text. This section will focus on the cumulative regret bounds at different rankness. We analyze the effectiveness of the algorithm from two perspectives. First, at three different levels of low-rankness, the upper bound of cumulative regret increases with the number of random explorations and stabilizes as the number of exploitation rounds grows, ultimately demonstrating ideal sublinear growth. Additionally, theoretical results indicate that cumulative regret also rises with increasing levels of low dimensionality, which is validated in the relevant Figure \ref{Fig.sub.7}. Second, as shown in Figure \ref{Fig.sub.8}, the ratio of cumulative regret to the theoretical bound stabilizes around a constant less than 1 as the number of rounds increases, indicating that our algorithm is consistent with the theoretical bound.

\subsection{Comparative Experiments under Lasso Bandits}
\begin{figure}[h]
    \centering
    \subfigure[$\rho=0.3,d=100,s=5$]{
        \label{Fig.sub.9}
        \includegraphics[width=0.20\textwidth]{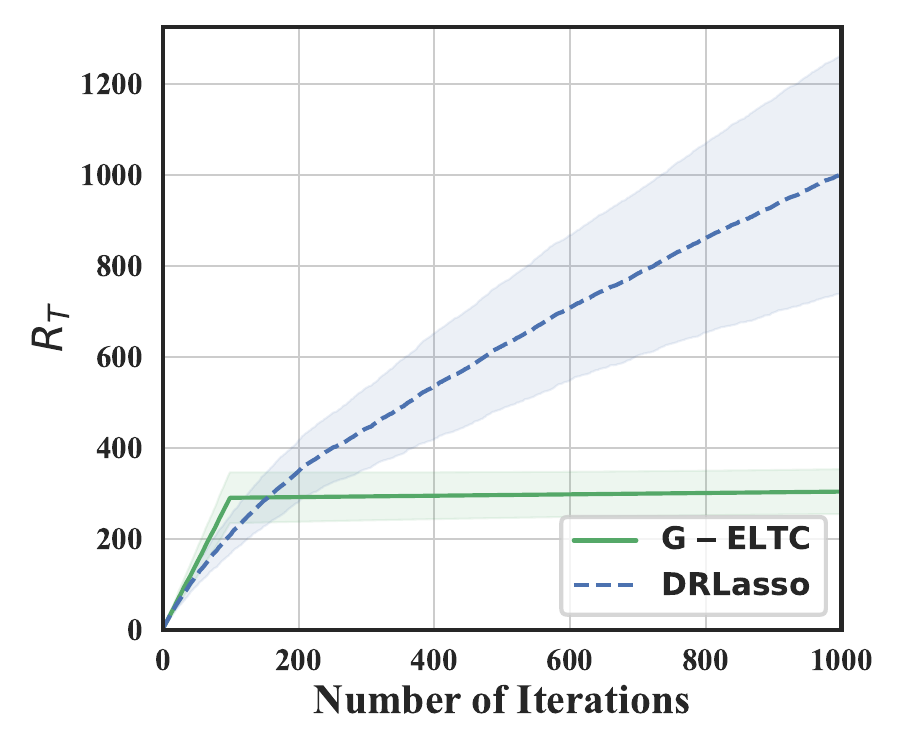}} 
        \hfil
    \subfigure[$\rho=0.7,d=100,s=5$]{
        \label{Fig.sub.10}
        \includegraphics[width=0.20\textwidth]{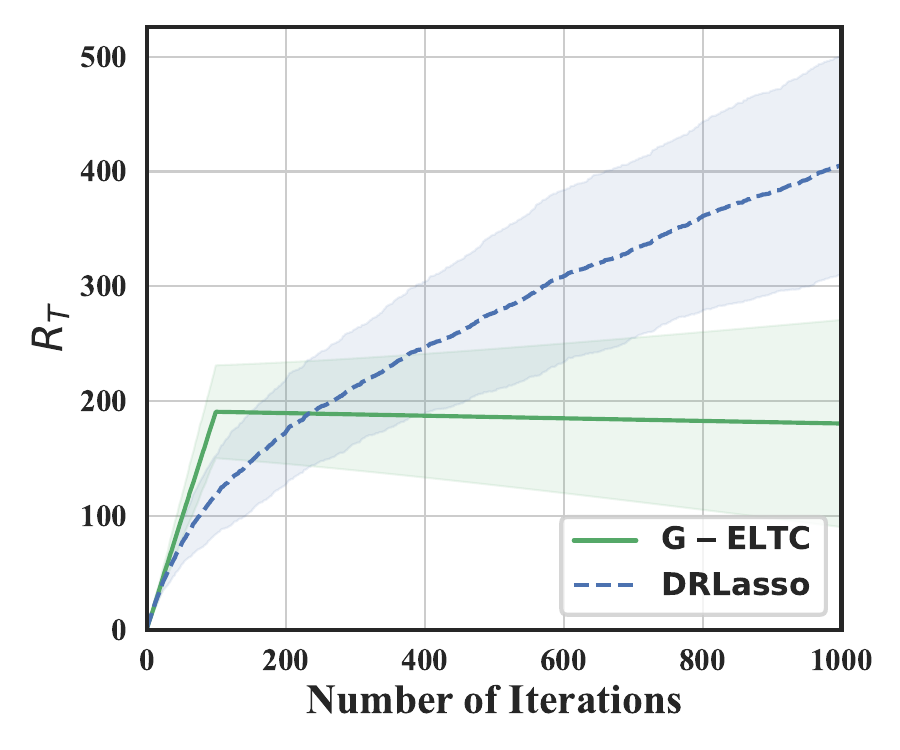}}
        \hfil
        \subfigure[$\rho=0.7,d=200,s=5$]{
        \label{Fig.sub.11}
        \includegraphics[width=0.20\textwidth]{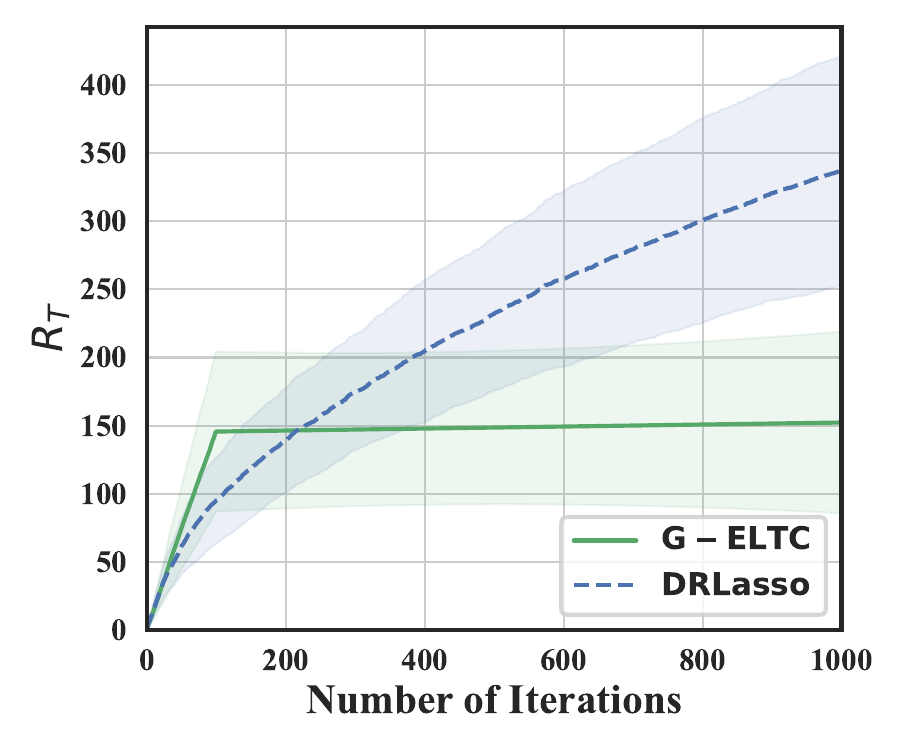}}
        \hfil
        \subfigure[$\rho=0.7,d=100,s=10$]{
        \label{Fig.sub.12}
        \includegraphics[width=0.20\textwidth]{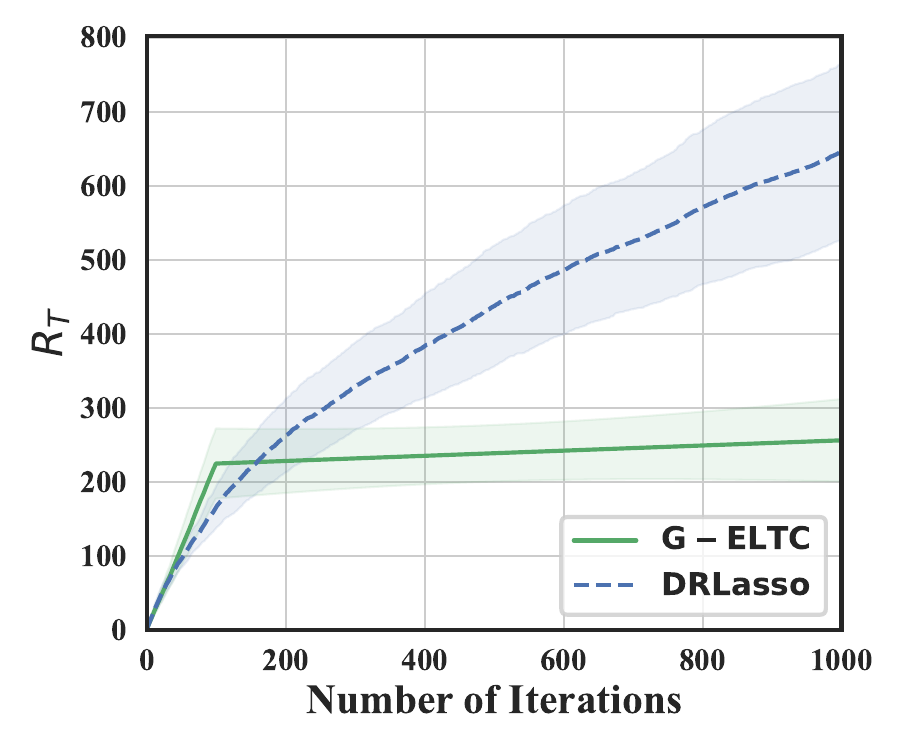}}
    \caption{Comparison of the cumulative regret bounds between the proposed algorithm and DR Lasso under the degenerate Lasso Bandit setting for different $\rho$, $d$, and $s$ configurations.}
    \label{fig:figure5}
\end{figure}


Considering that the algorithm proposed in this paper can degenerate into the Lasso bandits, to highlight the advantages of our algorithmic framework, we compare it with  doubly-robust (DR) Lasso bandits \citep{kim2019doubly}. The experimental setup is consistent with that in Section 5 of \cite{kim2019doubly}, where the number of arms $K = 100$, the dimension $d = 100,200$, and the sparsity $s = 5,10$. We conduct 10 replications for each case. The generation method of synthetic data is as follows:
\begin{itemize}
    \item The contexts of arms: $X \in \mathbb{R}^{K \times d}$, where $X_{\cdot j} \sim \mathcal{N}(0_K,V)$, $V(i,i)=1$ for every $i$ and $V(i,j)=\rho^{2}$ for every $i\neq j$. We experiment two cases for $\rho^2$, either $\rho^2=0.3$ (weak correlation) or $\rho^2=0.7$ (strong correlation). 
    \item Real parameter: For the sparse true parameters, we first randomly select $s$ non-zero indices, and then let these non-zero elements independently follow a uniform distribution on $[0,1]$.
    \item Reward: $y_t \sim \mathcal{N}(\langle x_t, \theta^* \rangle, R^2)$, where $R=0.05$.
\end{itemize}

In the case of degenerating to the Lasso bandit, we demonstrate the superiority of the framework proposed in this paper by comparing it with another algorithm, DR Lasso. There are two reasons for choosing DR Lasso for comparison: first, the settings of the algorithm are essentially the same as those in this paper, i.e., different arms share the same parameters; second, the two algorithms have their respective advantages and disadvantages in terms of theoretical bounds. Specifically, both algorithms achieve a logarithmic order bound for dimensionality, with DR Lasso having an advantage in the order of rounds, while it is inferior to the algorithm proposed in this paper in terms of sparsity. To ensure the credibility of the experimental results, the experimental settings strictly follow those described in the original DR Lasso paper. As shown in Figure \ref{fig:figure5}, in four different settings, the algorithm proposed in this paper achieves faster sublinear convergence in cumulative regret and obtains a lower cumulative regret. Furthermore, our algorithm is less sensitive to the correlation parameters.

\end{document}